%% file: ARXIV_version.tex
\documentclass[11pt]{article}
\let\chapter\section
\usepackage[ruled]{algorithm2e}
\usepackage[margin=1in]{geometry}
\usepackage{tikz-qtree}

\usepackage{amsthm,amsmath,amsfonts,amssymb,bbm}
\usepackage{mathtools}
\usepackage{varwidth}
\usepackage{caption}
\usepackage{subcaption}
\usepackage{breakcites}
\usepackage{etoolbox}
\usetikzlibrary{positioning}

\SetArgSty{textrm}  
\SetAlFnt{\small}
\SetAlCapFnt{\small}
\SetAlCapNameFnt{\small}
\SetAlCapHSkip{0pt}
\IncMargin{-\parindent}

\def\R{{\mathbb{R}}}
\newcommand{\defeq}{\vcentcolon=}

\newcommand{\sample}{\mathcal{S}}
\newtheorem{claim}{Claim}

\newtheorem{theorem}{Theorem}
\newtheorem{definition}{Definition}
\newtheorem{lemma}{Lemma}
\newtheorem{remark}{Remark}

\allowdisplaybreaks

\newtheoremstyle{named}{}{}{\itshape}{}{\bfseries}{.}{.5em}{\thmnote #1{ #3}}
\theoremstyle{named}
\newtheorem*{namedtheorem}{Theorem}

\begin{document}

\title{Sample Complexity of Automated Mechanism Design
\footnote{Authors' addresses: Carnegie Mellon University, School of Computer Science. Email: \texttt{\{ninamf,sandholm,vitercik\}@cs.cmu.edu}.
}}
\author{Maria-Florina Balcan \and
Tuomas Sandholm \and
Ellen Vitercik}

\maketitle

\begin{abstract}
\input{abstract}
\end{abstract}

\input{intro}

\input{summaryResults}

\input{relatedWork}

\input{prelim}

\input{AMA}
\input{MBA}
\input{conclusion}

\bibliography{dairefs}
\bibliographystyle{apalike}

\appendix
\section*{APPENDIX}
\setcounter{section}{0}
\input{appendix}

\end{document}

%% file: abstract.tex
The design of revenue-maximizing combinatorial auctions, i.e. multi-item auctions over bundles of goods, is one of the most fundamental problems in computational economics, unsolved even for two bidders and two items for sale. In the traditional economic models, it is assumed that the bidders' valuations are drawn from an underlying distribution and that the auction designer has perfect knowledge of this distribution. Despite this strong and oftentimes unrealistic assumption, it is remarkable that the revenue-maximizing combinatorial auction remains unknown. In recent years, \emph{automated mechanism design} has emerged as one of the most practical and promising approaches to designing high-revenue combinatorial auctions. The most scalable automated mechanism design algorithms take as input \emph{samples} from the bidders' valuation distribution and then search for a high-revenue auction in a rich auction class. In this work, we provide the first sample complexity analysis for the
standard hierarchy of deterministic combinatorial auction classes used in
automated mechanism design.  In particular, we provide tight sample complexity bounds on the
number of samples needed to guarantee that the empirical revenue of the
designed mechanism on the samples is close to its expected revenue on the
underlying, unknown distribution over bidder valuations, for each of the
auction classes in the hierarchy.  In addition to helping set automated
mechanism design on firm foundations, our results also push the boundaries
of learning theory. In particular, the hypothesis functions used in our
contexts are defined through multi-stage combinatorial optimization
procedures, rather than simple decision boundaries, as are common in machine
learning.

%% file: intro.tex
\section{Introduction}\label{sec:intro}

Multi-item, multi-bidder auctions have been studied extensively in economics, operations research, and computer science. In a \emph{combinatorial auction
(CA)}~\cite{Cramton06:Combinatorial}, the bidders may submit bids on bundles of goods, rather than on individual items alone, and thereby they may fully express their complex valuation functions. Notably, these functions may be non-additive due to the presence of complementary or substitutable goods for sale. There are many important and practical applications of CAs, ranging from the US government's wireless spectrum license auctions to sourcing auctions, through which companies coordinate the procurement and distribution of equipment, materials and supplies \cite{Cramton06:Combinatorial}.

One of the most important and tantalizing open questions in computational economics is the design of {\em optimal auctions}, that is,
auctions that maximize the seller's expected
revenue~\cite{Vohra:01}. In the standard economic model, it is assumed that the bidders' valuations are drawn from an underlying distribution and that the mechanism designer has perfect information about this distribution. Astonishingly, even with this strong assumption, the optimal CA design problem is unsolved even for
auctions with just two distinct items for sale and
two bidders.  A monumental advance in the study of optimal auction design was the characterization of the optimal 1-item auction~\cite{Myerson81:Optimal}.
In that auction, the winner and the payment are determined not based on the bids, but rather on \emph{virtual valuations} which are transformations of the bids in a way that makes weak bidders (\emph{i.e.}, bidders who are likely to have low valuations) artificially  more competitive.
That auction was later extended to the case of selling multiple copies of the same
item~\cite{Maskin89:Optimal}. However, the characterization of
revenue-maximizing \emph{multi-item} auctions has been obtained only for special
cases of the two-item two-bidder setting~\cite{Avery00:Bundling,Armstrong00:Optimal}.

While it might be surprising that the revenue-maximizing CA is unknown,
we observe that this is actually what one should expect once one views the problem through a computational lens. Conitzer and Sandholm proved that the problem of finding a revenue-maximizing CA (among all deterministic CAs with discrete types) is NP-complete \cite{Conitzer04:Self}.  Therefore, it is unlikely that a concise
characterization
of revenue-maximizing CAs (among deterministic CAs) can even \emph{exist}
\footnote{It is well known that randomization can increase revenue beyond that of the best deterministic CA.  In this paper we focus on deterministic CAs because randomized CAs 1) have \emph{ex post} fairness problems that can be unpalatable to bidders, 2) are harder for bidders and auctioneers to understand, and 3) are not used in practice, to our knowledge.}.

In recent years, a novel approach known as \emph{automated mechanism design} (AMD) has been adopted to attack the revenue-maximizing auction design problem \cite{Conitzer02:Mechanism,Sandholm03:Automated}. In one strand of AMD research, the support of the distribution of the bidders' valuations is discretized and the input to the design algorithm is a probability for each support point~\cite{Conitzer02:Mechanism,Sandholm03:Automated,Conitzer04:Self}. This has the challenge that the input is doubly exponential in the number of items. In an independent-private-values setting, the number of support points is $nk^{2^m}$, where $n$ is the number of bidders, $k$ is the number of discrete value levels a bidder can assign to a bundle, and $m$ is the number of items. This is because each of the $2^m$ bundles can take any of $k$ values.  With correlated valuations, the prior has $k^{2^{nm}}$ support points.  Therefore, that strand is not scalable~\cite{Conitzer03:Applications}, and it is unlikely that such priors are available in practical applications.

In contrast, in the most scalable strand of AMD research, algorithms have been developed which take samples from the bidders' valuation distributions as input, optimize over a rich class of auctions, and return an auction which is high-performing over the sample~\cite{Likhodedov04:Boosting,
Likhodedov05:Approximating,Sandholm15:Automated}. AMD algorithms have yielded deterministic mechanisms with the highest known revenues in the contexts used for empirical evaluations \cite{Sandholm15:Automated}. 
This approach relaxes the unrealistic assumption that the mechanism designer has perfect information about the bidders' valuation distribution.

However, until now, there was no formal characterization of the number of samples required to guarantee that the empirical revenue of the designed mechanism on the samples is close to its expected revenue on the underlying, unknown distribution over bidder valuations. In this paper, we provide that missing link. We present tight sample complexity guarantees over an extensive hierarchy of expressive CA families. These are the most commonly used auction families in AMD. The classes in the hierarchy are based on the classic VCG mechanism~\cite{Vickrey61:Counterspeculation,Clarke71:Multipart,Groves73:Incentives}, which is a generalization of the well-known second-price, or Vickrey, single-item auction. The auctions we consider achieve significantly higher revenue than the VCG baseline by weighting bidders (multiplicatively increasing all of their bids) and boosting outcomes (additively increasing the liklihood that a particular outcome will be the result of the auction).

A major strength of our results is their applicability to any algorithm that determines the optimal auction over the sample, a nearly optimal approximation, or any other black box procedure. Therefore, they apply to any automated mechanism design algorithm, optimal or not. One of the key challenges in deriving these general sample complexity bounds is that to do so, we must develop deep insights into how changes to the auction parameters (the bidder weights and allocation boosts)
effect the outcome of the auction (who wins which items and how much each bidder pays) and thereby the revenue of the auction. In our context, we show that the functions which determine the outcome of an auction are highly complex, consisting of multi-stage optimization procedures.

Therefore, the function classes we consider are much more challenging than those commonly found in machine learning contexts. Typically, for well-understood classes of functions used in machine learning, such as linear separators or other smooth curves in Euclidean spaces, there is a simple
mapping from the parameters of a specific hypothesis to its prediction on a given
example and a close connection between the distance in the parameter space
between two parameter vectors and the distance in function space between
their associated hypotheses. Roughly speaking, it is necessary to understand this connection in order to determine
how many significantly different hypotheses there are over the full range of parameters. In our context, due to the inherent complexity of the classes we consider, connecting the parameter space to the space of revenue functions requires a
much more delicate analysis. Indeed, the key technical part of our work
involves understanding this connection from a learning theoretic perspective. For the more general classes in the hierarchy, we use Rademacher complexity to derive our bounds, and for the auction classes with more combinatorial structure, we exploit that structure to prove pseudo-dimension bounds. Therefore, this work is both of practical importance since we fill a fundamental gap in AMD, and of learning theoretical interest, as our sample complexity analysis requires a deep understanding of the structure of the revenue function classes we consider.

\subsection{The Hierarchy of Deterministic Combinatorial Auctions}

Early work in automated mechanism design approached the mechanism design problem as an integer program or linear program~\cite{Conitzer02:Mechanism,Sandholm03:Automated,Conitzer04:Self,Conitzer03:Applications}. Then a more scalable approach emerged where the design focuses on a parameterized family of CA mechanisms. In that approach, the design of a high-revenue CA is conducted via an algorithmic search for a good parameter vector within the family~\cite{Likhodedov04:Boosting,Likhodedov05:Approximating,Sandholm15:Automated}. Under this view, there is a hierarchy of CA families which we will now describe, and which is depicted in Figure~\ref{fig:hierarchy}.
We define these families formally in Section~\ref{sec:prelim}.
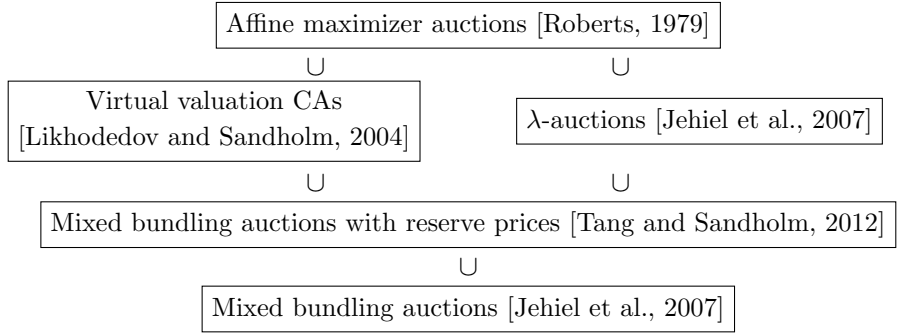
\begin{figure}
\centerline{\begin{tikzpicture}
[place/.style={rectangle,draw},
transition/.style={rectangle,draw},
phantom/.style={rectangle}]
\node[place] (AMA)[label=-10:$\cup$,label=-170:$\cup$] {\begin{varwidth}{\textwidth}\begin{center} {\small Affine maximizer auctions  \cite{Roberts79:Characterization}}\end{center}\end{varwidth}};
\node[phantom] (phantom)[below = .83cm of AMA] {};
\node[transition] (VVCA) [left = .5cm of phantom] {{\begin{varwidth}{\textwidth}\begin{center} {\small Virtual valuation CAs\\ \cite{Likhodedov04:Boosting}}\end{center}\end{varwidth}}};
\node[transition] (lambda) [right = .5cm of phantom] {{\begin{varwidth}{5cm}\begin{center} {\small $\lambda$-auctions \cite{Jehiel07:Mixed}}\end{center}\end{varwidth}}};
\node[place] (MBARP) [below = .93cm of phantom,label=10:$\cup$,label=170:$\cup$] {{\begin{varwidth}{\textwidth}\begin{center} {\small Mixed bundling auctions with reserve prices \cite{Tang12:Mixed}}\end{center}\end{varwidth}}};
\node[place] (MBA) [below = .5cm of MBARP, label=90:$\cup$] {{\begin{varwidth}{\textwidth}\begin{center} {\small Mixed bundling auctions \cite{Jehiel07:Mixed}}\end{center}\end{varwidth}}};
\end{tikzpicture}}
\caption{The hierarchy of deterministic combinatorial auctions families.  Generality increase upward in the hierarchy.}
\label{fig:hierarchy}
\end{figure}

The most general family in the hierarchy of deterministic combinatorial auctions that we study is \emph{affine maximizer auctions (AMAs)}~\cite{Roberts79:Characterization}. It contains the VCG mechanism as a special case, as well auctions that achieve higher expected revenue than the VCG by weighting bidders and boosting allocations. In particular, if the weight of a bidder is increased, any bid she submits will be increased multiplicatively by that amount. If an allocation is boosted by adding a monetary preference to it,  the chance that it will be the AMA allocation is increased.
The parameters of an AMA are the coefficients in these bidder weightings and allocation boostings.

In the classes below AMAs in the hierarchy, more constraints are added to these transformations, thereby decreasing the flexibility of the auctions. For example, the class of \emph{virtual valuation combinatorial auctions (VVCAs)}~\cite{Likhodedov04:Boosting} consists of AMAs with a restricted set of allowable allocation boosts; the structure is such that the parameters can be thought of as affine transformation parameters of each bidder's valuation function---hence the name of the family. Meanwhile, in a \emph{$\lambda$-auction}~\cite{Jehiel07:Mixed}, any allocation boost is valid, but no bidder is weighted more than any other. In a \emph{mixed bundling auction (MBA)}~\cite{Jehiel07:Mixed}, the only allowed allocation boosts are for those wherein a single bidder receives all of the items in the auction. These auctions can be supplemented with reserve prices, which yields the family of \emph{mixed bundling auction with reserve prices (MBARPs)}~\cite{Tang12:Mixed}.

%% file: summaryResults.tex
\subsection{Summary of Results and Techniques}


For each family in the hierarchy, we prove  strong upper bounds on the number of samples required to guarantee that with high probability, for any auction in the family, the expected revenue of the auction is close to the average revenue over the samples. In learning-theoretic terms, these are called uniform convergence sample complexity bounds and they have the nice feature that they apply to any procedure one might use to optimize over the samples, such as an algorithm that returns the optimal auction over the sample or a nearly optimal approximation, as well as any other black box procedure. Note that given any two auction families such that one of them is a subset of the other, the  uniform convergence sample complexity bound for the smaller family is always upper bounded by the  uniform convergence sample complexity bound of the larger one. Therefore the sample complexity results we obtain for AMAs immediately apply to its subfamilies. For these subfamilies, however, we exploit their unique structures and thus derive even better upper bounds.

We will now summarize our main results a bit more formally. Let $\mathcal{A}$ be a fixed class of auctions (e.g. AMAs or VVCAs) and define $rev_A(\vec{v})$ to be the revenue of an auction $A \in \mathcal{A}$ on a vector of bidder valuations $\vec{v}$. Given a distribution $\mathcal{D}$, $\mathbb{E}_{\vec{v} \sim \mathcal{D}}[rev_A(\vec{v})]$ is the expected revenue of the auction on a vector of bidder valuations drawn at random from $\mathcal{D}$. Moreover, given a sample $\sample = \left\{\vec{v}^1, \dots, \vec{v}^N\right\}$ of bidder valuations, $\frac{1}{N} \sum_{i = 1}^N rev_A\left(\vec{v}^i\right)$ is the average revenue of $A$ over the sample. Now, we define the \emph{sample complexity of uniform convergence over $\mathcal{A}$} as follows.

\begin{definition}[Sample complexity of uniform convergence over $\mathcal{A}$]
We say that $N(\epsilon, \delta, \mathcal{A})$ is the sample complexity of uniform convergence over $\mathcal{A}$ if for any $\epsilon, \delta \in (0,1)$, if $\sample = \left\{\vec{v}^1, \dots, \vec{v}^N\right\}$ is a sample of size $N \geq N(\epsilon, \delta, \mathcal{A})$ drawn at random from $\mathcal{D}$, with probability at least $1-\delta$, for all auctions $A \in \mathcal{A}$, $\left|\frac{1}{N} \sum_{i = 1}^N rev_A\left(\vec{v}^i\right) - \mathbb{E}_{\vec{v} \sim \mathcal{D}}\left[rev_A(\vec{v})\right]\right| \leq \epsilon$.
\end{definition}

In other words, the sample complexity of uniform convergence over $\mathcal{A}$ is the sufficient number of samples such that uniformly for all auctions in that class, the expected revenue over the distribution is close to the average revenue over the sample.

In Theorem~\ref{thm:main_sample_AMA}, we bound the sample complexity of uniform convergence for the classes of AMAs, VVCAs, and $\lambda$-auctions, and we prove lower bounds with near-tight dependence on the number of bidders $n$ and the number of items $m$. We go on to prove tighter upper bounds for the restricted classes of MBAs and MBARPs in Theorems~\ref{thm:main_sample_MBARP} and \ref{thm:main_sample_MBAn}.

These upper bounds immediately imply that for any algorithm that outputs the auction $A$ that achieves maximum average revenue over its input samples, we can guarantee that the expected revenue of $A$ is close to the expected revenue of the best auction with respect to the actual---unknown---distribution. In particular, for a fixed class of auctions $\mathcal{A}$, suppose that $\hat{A}$ is the auction that maximizes average revenue over the samples $\sample$ and $A^*$ is the auction that maximizes expected revenue with respect to the distribution $\mathcal{D}$. The sample complexity of uniform convergence over $\mathcal{A}$ is sufficient to ensure that with high probability, the expected revenue of $\hat{A}$ is close to the expected revenue of $A^*$. In other words, we can guarantee that if we learn the best auction over the sample, then it will achieve almost maximal revenue with respect to the best auction in that class.

We are now ready to present our main results. For a fixed class of auctions and domain $X$ over the bidders' valuation functions, let $rev_A$ be the corresponding revenue function of an auction $A$ in that class, where $rev_A : X \to [0,U]$ for some $U \in \R$.

\begin{theorem}\label{thm:main_sample_AMA}
The sample complexity of uniform convergence over the classes of $n$-bidder, $m$-item AMAs, VVCAs, and $\lambda$-Auctions is
\[N = \widetilde{O} \left(\left[\frac{U}{\epsilon}n^m\sqrt{m} \left( U+ n^{m/2}\right)\right]^2\right).\]
Moreover, for $\lambda$-Auctions, $N = \Omega\left(n^m\right)$ and for VVCAs, $N= \Omega\left(2^m\right)$.
\end{theorem}

We prove Theorem~\ref{thm:main_sample_AMA} by splitting the complex AMA revenue function into $n+1$ simpler and economically coherent pieces: the maximum weighted social welfare without any one bidder’s participation and the amount of revenue subtracted out to ensure the resulting auction is \emph{strategy-proof} (defined in Section~\ref{sec:prelim}. We analyze these simpler functions using Rademacher complexity, a tool from learning theory, and combine these analyses using compositional properties of Rademacher complexity to bound the sample complexity of the function class as a whole.

\begin{theorem}\label{thm:main_sample_MBARP}
The sample complexity of uniform convergence over the class of $n$-bidder, $m$-item MBARPs with item-specific reserve prices is \[N = O\left(\left(\frac{U}{\epsilon}\right)^2\left( m^3 \log n \log\frac{U}{\epsilon} + \log \frac{1}{\delta}\right)\right).\]
\end{theorem}
\begin{theorem}\label{thm:main_sample_MBAn}
The sample complexity of uniform convergence over the class of $n$-bidder, $m$-item MBAs is \[N = O\left(\left(\frac{U}{\epsilon}\right)^2\left( \log\frac{U}{\epsilon} + \log \frac{1}{\delta}\right)\right).\]

\end{theorem}

To prove Theorems~\ref{thm:main_sample_MBARP} and \ref{thm:main_sample_MBAn}, we characterize the mapping from the MBA (and MBARP) parameter space to the revenue of the associated auctions on an arbitrary bidding instance. We then use these structural insights to prove bounds on the pseudo-dimension of these revenue functions, another learning-theoretic tool which allows us to derive strong sample complexity bounds.

In our bounds, we observe the usual dependence on $U$, which is necessary when analyzing the sample complexity of learning over real-valued functions because it measures the extent to which a single example can influence the average function value over the sample.

We note that it might not always be \emph{computationally} feasible to solve for the best auction over $\sample$ for the given auction family. Rather, we may only be able to design an auction $A$ within the family that has average revenue over $\sample$ that is within a $(1+\alpha)$ multiplicative factor of the revenue-maximizing auction over $\sample$ within the family. Nonetheless, in Theorem~\ref{thm:approx_bounds} we prove that with slightly more samples, we can ensure that the expected revenue of $A$ is close to being with a $(1+\alpha)$ multiplicative factor of the expected revenue of the optimal auction within the family with respect to the real---unknown---distribution $\mathcal{D}$. We prove a similar bound for an additive factor approximation as well. Formally, we prove the following result, which holds very generally for any function class $\mathcal{H}$ with domain $X$ and for any arbitrary loss function $\ell:\mathcal{H} \times X \to [-c,c]$ for some $c \in \R$.

\begin{theorem}\label{thm:approx_bounds}
Let $\sample = \left\{x_1, \dots, x_N\right\}$ be a sample drawn from $\mathcal{D}$ and $\epsilon, \delta \in (0,1)$ be given. Suppose that $N$ is sufficiently large to ensure that with probability at least $1-\delta/2$, for any $h \in \mathcal{H}$, $\underset{x \sim \mathcal{D}}{\mathbb{E}}\left[\ell\left(h, x\right)\right] -  \frac{1}{N}\sum_{i = 1}^N \ell\left(h, x_i\right) < \epsilon$.

Suppose $h^* \in \mathcal{H}$ is a function that minimizes expected loss with respect to the distribution, $\hat{h}$ is a function that minimizes average loss over the sample $\sample$, and $\tilde{h} \in \mathcal{H}$ is a function such that the average loss of $\hat{h}$ over $\sample$ is within an additive $\rho$ factor of the average loss of $\tilde{h}$ over $\sample.$ In other words, $\frac{1}{N}\sum_{i = 1}^N \ell\left(\tilde{h}, x_i\right) -\frac{1}{N}\sum_{i = 1}^N \ell\left(\hat{h}, x_i\right) \leq \rho$ for some $\rho > 0$. Then with probability at least $1-\delta$, \[\underset{x \sim \mathcal{D}}{\mathbb{E}}\left[\ell\left(\tilde{h}, x\right)\right] - \underset{x \sim \mathcal{D}}{\mathbb{E}}\left[\ell\left(h^*, x\right)\right] \leq \epsilon + c\sqrt{\frac{\ln(4/\delta)}{2N}} + \rho.\]

Meanwhile, if $\frac{1}{N}\sum_{i = 1}^N \ell\left(\tilde{h}, x_i\right) \leq (1+\alpha)\frac{1}{N}\sum_{i = 1}^N \ell\left(\hat{h}, x_i\right)$ for some $\alpha \in [0,1)$, then \[\underset{x \sim \mathcal{D}}{\mathbb{E}}\left[\ell\left(\tilde{h}, x\right)\right] - \underset{x \sim \mathcal{D}}{\mathbb{E}}\left[\ell\left(h^*, x\right)\right] \leq \epsilon + (1+\alpha)\left(c\sqrt{\frac{\ln(4/\delta)}{2N}}\right) + \alpha \underset{x \sim \mathcal{D}}{\mathbb{E}}\left[\ell\left(h^*, x\right)\right].\]

Moreover, both bounds are tight in the worst case.
\end{theorem}

For any class of auctions $\mathcal{A}$ and corresponding set of revenue functions $\mathcal{H}_{\mathcal{A}} = \left\{rev_A \ | \ A \in \mathcal{A}\right\}$, this result can be instantiated by setting the loss function such that $\ell\left(rev_A, \vec{v}\right) = -rev_A\left(\vec{v}\right)$, and therefore by minimizing loss, we are maximizing revenue. The proof of Theorem~\ref{thm:approx_bounds} can be found in Appendix~\ref{app:intro_app}.

%% file: relatedWork.tex
\subsection{Additional Related Research}

In prior research, most analyses of the revenue achieved by the classes that make up the hierarchy of deterministic CAs have been empirical~\cite{Sandholm03:Automated,
Likhodedov04:Boosting,
Likhodedov05:Approximating,Tang12:Mixed,Sandholm15:Automated}. However, from a theoretical standpoint, Roberts, when introducing the class of AMAs~\cite{Roberts79:Characterization}, proved that they are the
only \emph{ex post} strategy-proof mechanisms over unrestricted
domains of valuations\footnote{A mechanism is \emph{ex post} strategy-proof if truthful bidding is an \emph{ex post Nash equilibrium} in which all bidders always receive nonnegative utility. By \emph{ex post Nash equilibrium}, we mean that for each player, no matter the valuations of the other players but given that they are bidding truthfully, that player will maximize her utility if she bids truthfully as well.}. Lavi et al. went on to prove that under certain
natural assumptions, every incentive compatible CA
is almost\footnote{A mechanism is an \emph{almost affine maximizer} if it is an affine
maximizer for sufficiently high valuations \cite{Lavi03:Towards}. Lavi et al. conjecture
that the ``almost'' qualifier is merely technical, and can be
removed in future research.} an affine maximizer.

In the intersection of learning theory and mechanism design, the sample complexity of revenue maximization has been studied primarily in the single-item or the more general single-dimensional settings~\cite{Elkind07:Designing,Cole14:Sample, Huang15:Making, Medina14:Learning, Morgenstern15:Pseudo, Roughgarden15:Ironing,Devanur16:Sample}, as well as some multi-dimensional settings which are reducible to the single-bidder setting \cite{Morgenstern16:Learning}. In contrast, the combinatorial settings that we study are much more complex since the revenue functions consist of multi-stage optimization procedures that cannot be reduced to a single-bidder setting. The complexity intrinsic to the multi-item setting is explored by Dughmi et al., who show that for a single unit-demand bidder, when the bidder's values for the items may be correlated, $\Omega(2^m)$ samples are required
to determine a constant-factor approximation to the optimal auction \cite{Dughmi14:Sampling}.

Learning theory tools such as pseudo-dimension and Rademacher complexity have been used to prove strong guarantees in auction settings \cite{Medina14:Learning,Morgenstern15:Pseudo,Morgenstern16:Learning}. These authors have analyzed certain classes of piecewise linear revenue functions and shown that few samples are needed to learn over these specific classes. In a similar direction, bounds on the sample complexity of welfare-optimal item pricings have been developed \cite{Feldman15:Combinatorial, Hsu16:Prices}.

Despite the inherent complexity of designing high-revenue CAs, Morgenstern and Roughgarden use linear separability as a tool to prove that certain simple classes of multi-parameter auctions have small sample complexity, such as sequential auctions with item and bundle pricings and second-price item auctions with reserve prices \cite{Morgenstern16:Learning}. In particular, they show that bounding the sample complexity of these sequential auctions can be reduced to the single-buyer setting. In contrast, the hierarchy we study consists of VCG-based mechanisms, as opposed to sequential auctions. These VCG-based revenue function classes are more versatile than item pricing auctions because they allow the mechanism designer many more degrees of freedom than the number of items. Moreover, even the simpler auction classes we consider pose a unique challenge because the parameters defining the auctions interact in non-intuitive ways with the multi-stage optimization procedures which define the revenue functions we work with, unlike item pricings, which are simple by design. Our function classes therefore require us to understand the specific form of the weighted VCG payment rule and its interaction with the parameter space. Thus, our context and techniques diverge from those in \cite{Morgenstern16:Learning}.

Earlier work of Balcan et al.\ addressed sample complexity results for revenue maximization in unrestricted supply settings \cite{BBHM}. The settings considered by Balcan et al. are significantly simpler to analyze since in the unrestricted supply settings, the hypothesis classes are straightforward to analyze from a learning theory perspective and the revenue function decomposes additively among bidders.

Finally, there is a wealth of work on characterizing the optimal CA for restricted settings and designing mechanisms which achieve high, if not optimal revenue in specific contexts. The simplicity of Myerson's optimal \emph{single-item} auction might lead one to hope that the optimal multi-item auction could be so elegantly characterizable~\cite{Myerson81:Optimal}. Recent work has made considerable progress toward this end~(e.g. \cite{Alaei13:Simple,Bhalgat13:Optimal,Bhattacharya10:Budget,
Cai12:Algorithmic,Cai12:Optimal,Cai13:Reducing,Daskalakis14:Complexity,
Kleinberg12:Matroid}) but there is still relatively little known about optimal multi-item auction design. The problem has also garnered significant interest from a more applied perspective, resulting in significant advances from the artificial intelligence and machine learning communities (e.g. \cite{Parkes00:Iterative,Lahaie11:Kernel,
Wurman00:AkBA,Parkes04:Approximately,Amin13:Learning,Mohri14:Optimal,Mohri15:Revenue}).

Revenue-maximizing mechanism design complements an active research area in theoretical computer science which strives to answer the question: can \emph{simple} mechanisms achieve near-optimal revenue? This question was posed by Hartline and Roughgarden, who left the precise definition of a simple mechanism open for interpretation \cite{Hartline09:Simple}. Recently, Morgenstern and Roughgarden proposed an auction class's pseudo-dimension as a formal means of defining simplicity \cite{Morgenstern15:Pseudo,Morgenstern16:Learning}. In particular, Morgenstern and Roughgarden complemented pseudo-dimension bounds with known approximation guarantees for the corresponding simple auction classes \cite{Morgenstern16:Learning}. See \cite{Morgenstern16:Learning} and references therein for descriptions of these guarantees.

%% file: prelim.tex
\section{Preliminaries}\label{sec:prelim}
In the following section, we explain the basic mechanism design problem, fix notation, and then describe the hierarchy of combinatorial auction families we study.

\subsection{Mechanism design background} We consider the problem of selling $m$ heterogeneous goods to $n$ bidders. This means that there are $2^m$ different bundles of goods, $B = \left\{b_1, \dots, b_{2^m}\right\}$. Each bidder $i \in [n]$ is associated with a set-wise valuation function over the bundles, $v_i: B \to \R$. We assume that the bidders' valuations are drawn from a distribution $\mathcal{D}$.

Every auction is defined by an \emph{allocation function} and a \emph{payment function}. The allocation function determines which bidders receive which items based on their bids and the payment function determines how much the bidders need to pay based on their bids and the allocation. It is up to the mechanism designer to determine which allocation and payment functions should be used. In our context, the two functions are fixed based on the samples from $\mathcal{D}$ before the bidders submit their bids.

Each auction family that we consider has a design based on the classic \emph{Vickrey-Clarke-Groves  mechanism (VCG)}. The VCG mechanism, which we describe below, is the canonical \emph{strategy-proof} mechanism, which means that every bidder's dominant strategy is to bid truthfully. In other words, for every Bidder $i$,
no matter the bids made by the other bidders, Bidder $i$ maximizes her expected utility (her value for her allocation minus the price she pays) by bidding her true value. Therefore, we describe the VCG mechanism assuming that the bids equal the bidders' true valuations.

The VCG mechanism allocates the items such that the social welfare of the bidders, that is, the sum of each bidder's value for the items she wins, is maximized. Intuitively, each winning bidder then pays her bid minus a ``rebate'' equal to the increase in
welfare attributable to her presence in the auction. This form of the payment function is crucial to ensuring that the auction is strategy-proof. More concretely, the allocation of the VCG mechanism is the disjoint set of subsets $\left(b^*_1, \dots, b^*_n\right) \subseteq B$ that maximizes $\sum v_i\left(b_{i}^*\right)$. Meanwhile, let $\left(b_1^{-i}, \dots, b_n^{-i}\right)$ be the disjoint set of subsets that maximizes $\sum_{j \not=i} v_j\left(b_j^{-i}\right)$. Then Bidder $i$ must pay $\sum_{j \not=i} \left[v_j\left(b_j^{-i}\right) - v_j\left(b_j^{*}\right)\right]=v_i\left(b_{i}^*\right) - \left[\sum v_j\left(b_j^{*}\right)- \sum_{j \not=i} v_j\left(b_j^{-i}\right)\right]$. In the special case where there is one item for sale, the VCG mechanism is known as the second price, or Vickrey, auction, where the highest bidder wins the item and pays the second highest bid. We note that every auction in the classes we study is strategy-proof, so we may assume that the bids equal the bidders' valuations.

\subsubsection{Notation}

 We study auctions with $n$ bidders and $m$ items. We refer to the bundle of all $m$ items as the \emph{grand bundle}. In total, there are $(n+1)^m$ possible allocations, which we denote as the vectors $\mathcal{O} = \left\{\vec{o}_1,\dots, \vec{o}_{(n+1)^m}\right\}.$ Each allocation vector $\vec{o}_i$ can be written as $\left(o_{i,1}, \dots, o_{i,n}\right)$, where $o_{i,j} = b_{\ell} \in B$ denotes the bundle of items allocated to Bidder $j$ in allocation $\vec{o}_i$.  We use the notation $\vec{v}_1 = \left(v_1\left(b_1\right), \dots, v_1\left(b_{2^m}\right)\right)$ and $\vec{v} = \left(\vec{v}_1, \dots, \vec{v}_n\right)$ to denote a vector of bidder valuation functions. We say that $rev_A(\vec{v})$ is the revenue of an auction $A$ on the valuation vector $\vec{v}$. Denoting the payment of any one bidder under auction $A$ given valuation vector $\vec{v}$ as $p_{i,A}\left(\vec{v}\right)$, we have that $rev_A(\vec{v}) = \sum_{i = 1}^n p_{i,A}\left(\vec{v}\right)$.

 \smallskip

\subsubsection{Auction classes} We now give formal definitions of the CA families in the hierarchy we study. See Figure~\ref{fig:hierarchy} for the hierarchical organization of the auction classes, together with the papers which introduced each family.

\smallskip

\noindent \textbf{Affine maximizer auctions (AMAs).} An AMA $A$ is defined by a set of weights per bidder $\left(w_1,\dots,w_n\right) \subset \R_{> 0}$ and boosts per allocation $\left(\lambda\left(\vec{o}_1\right),\dots, \lambda\left(\vec{o}_{(n+1)^m}\right)\right) \subset \R$. An auction $A$ uniquely corresponds to a set of these parameters, so we write $A = \left(w_1,\dots,w_n,\lambda\left(\vec{o}_1\right),\dots, \lambda\left(\vec{o}_{(n+1)^m}\right)\right)$. To simplify notation, we write  $\lambda_i = \lambda\left(\vec{o}_i\right)$ interchangeably. These parameters allow the mechanism designer to multiplicatively boost any bidder's bids by their corresponding weight and to increase the likelihood that any one allocation is returned as the output of an auction. More concretely, the allocation of an AMA $A$, is $\vec{o}^* = \text{argmax}_{\vec{o}_i \in \mathcal{O}} \left\{\sum_{j = 1}^n w_jv_j\left(o_{i,j}\right) + \lambda\left(\vec{o}_i\right)\right\}.$ The payment function of $A$ has the same form as the VCG payment rule, with the parameters factored in to ensure that the auction remains strategy-proof. In particular, for all $j \in [n]$, the payments are $p_{j,A}\left(\vec{v}\right) = \frac{1}{w_j}\left[ \sum_{\ell \not= j} w_{\ell}v_{\ell}\left(o_{-j,\ell}\right) + \lambda\left(\vec{o}_{-j}\right) - \sum_{\ell \not= j} w_{\ell} v_{\ell}\left(o^*_\ell\right) -\lambda\left(\vec{o}^*\right)\right],$ where $\vec{o}_{-j} = \text{argmax}_{\vec{o}_i \in \mathcal{O}} \left\{ \sum_{\ell \not= j} w_{\ell}v_{\ell}\left(o_{i,\ell}\right) + \lambda\left(\vec{o}_i\right)\right\}.$ We assume that $H_{\underline{w}} \leq w_i \leq H_{\overline{w}}$, $\lambda_i \leq H_\lambda$, and $v_i\left(b_\ell\right) \leq H_v$ for some $H_{\underline{w}},H_{\overline{w}},H_\lambda, H_v \in \R_{\geq 0}$.

\smallskip

\noindent \textbf{Virtual valuation combinatorial auctions (VVCAs).}
VVCAs are a subset of AMAs. The defining characteristic of a VVCA is that each $\lambda\left(\vec{o}_j\right)$ is split into $n$ terms such that $\lambda\left(\vec{o}_j\right) = \sum_{i = 1}^n \lambda_i\left(\vec{o}_j\right)$ where $\lambda_i\left(\vec{o}_j\right) = c_{i,b}$ for all allocations $\vec{o}_j$ that give Bidder $i$ exactly bundle $b \in B$.

\smallskip

\noindent \textbf{$\lambda$-auctions.}
$\lambda$-auctions are the subclass of AMAs where $w_i = 1$ for all $i \in [n]$.

\smallskip

\noindent \textbf{Mixed bundling auctions (MBAs).}
The class of MBAs is parameterized by a constant $c \geq 0$ which can be seen as a discount for any bidder who receives the grand bundle. Formally, the $c$-MBA is the $\lambda$-auction with $\lambda(\vec{o}) = c$ if some bidder receives the grand bundle in allocation $\vec{o}$ and 0 otherwise.

\smallskip
\noindent \textbf{Mixed bundling auctions with reserve prices (MBARPs).}
MBARPs are identical to MBAs though with \emph{reserve prices.} In a single-item VCG auction (i.e. second price auction) with a reserve price, the item is only sold if the highest bidder's bid exceeds the reserve price, and the winner must pay the maximum of the second highest bid and the reserve price. To generalize this intuition to the multi-item case, we enlarge the set of agents to include the seller, who is now Bidder 0 and whose valuation for a set of items is the set's reserve price. Working in this expanded set of agents, the bidder weights are all 1 and the $\lambda$ terms are the same as in the standard MBA setup. Importantly, the seller makes no payments, no matter her allocation. More formally, given a vector of valuation functions $\vec{v}$, the MBARP allocation is $\vec{o}^* = \text{argmax}_{\vec{o} \in \mathcal{O}} \sum_{i = 0}^n v_i\left(o_i\right) + \lambda\left(\vec{o}\right).$ For each $i \in \{1, \dots, n\}$, Bidder $i$'s payment is \[p_{A,i}(\vec{v}) = \sum_{j \in \{0, \dots, n\}\setminus \{i\}} v_j\left(o_{-i,j}\right) + \lambda\left(\vec{o}_{-i}\right) - \sum_{j \in \{0, \dots, n\}\setminus \{i\}} v_j\left(o^*_j\right) - \lambda\left(\vec{o}^*\right),\] where \[\vec{o}_{-i} = \underset{\vec{o} \in \mathcal{O}}{\text{argmax}} \sum_{j \in \{0, \dots, n\}\setminus \{i\}} v_j\left(o_j\right) + \lambda\left(\vec{o}\right).\]


\subsection{Computational learning theory background}\label{sec:techniques}

To derive the upper bounds in Theorems~\ref{thm:main_sample_AMA} through \ref{thm:main_sample_MBAn}, we use two learning-theoretic tools which quantify the ``complexity'' of a class of functions: Rademacher complexity and pseudo-dimension. We define these concepts generally for a class of functions $\mathcal{H}$ with domain $X$ and distribution $\mathcal{D}$ over $X$. Further, we define $\ell$ to be an arbitrary loss function mapping $\mathcal{H} \times X$ to $[-c,c]$ for some $c \in \R$. To simplify notation, we let $\mathcal{F} \defeq \ell \circ \mathcal{H} \defeq \left\{x \mapsto \ell(h,x) \ | \ h \in \mathcal{H}\right\}$.

\subsubsection{Rademacher Complexity}
First, we formally define Rademacher complexity, which is somewhat technical, and then provide a more intuitive notion of the quantity that it measures.

\begin{definition}[Empirical Rademacher complexity] The \emph{empirical Rademacher complexity} of $\mathcal{F}$ with respect to the sample $\sample = \left\{x_1, \dots, x_N\right\}$ is defined as: $\widehat{\mathcal{R}}_\sample(\mathcal{F}) = \mathbb{E}_{\vec{\sigma}}\left[\sup_{f \in \mathcal{F}} \frac{1}{N} \sum_{i = 1}^N \sigma_i \cdot f\left(x_i\right) \right],$ where $\vec{\sigma} = (\sigma_1,\dots, \sigma_N)^\top$, with $\sigma_i$s independent uniform random variables taking values in $\{-1,1\}$. The random variables $\sigma_i$ are called \emph{Rademacher variables}.
\end{definition}

\begin{definition}[Rademacher complexity]
For any integer $N \geq 1$, the \emph{Rademacher complexity} of $\mathcal{F}$ is the expectation of the empirical Rademacher complexity over all samples of size $N$ drawn according to $\mathcal{D}$, i.e. $\mathcal{R}_N(\mathcal{F}) = \mathbb{E}_{\sample \sim \mathcal{D}^N}\left[\widehat{\mathcal{R}}_\sample(\mathcal{F})\right].$
\end{definition}

%

Intuitively, the supremum measures, for a given sample $\sample$ and Rademacher vector $\vec{\sigma}$, the maximum correlation between $f(x_i)$ and $\sigma_i$ over all $f \in \mathcal{F}$. Taking the expectation over $\vec{\sigma}$, we can then say that the empirical Rademacher complexity of $\mathcal{F}$ measures the ability of functions from $\mathcal{F}$ (when applied to a fixed sample $\sample$) to fit random noise. The Rademacher complexity of $\mathcal{F}$ therefore measures the expected noise-fitting-ability of $\mathcal{F}$ over all data sets $\sample \in X^N$ that could be drawn according to the distribution $\mathcal{D}$.

We are able to derive strong sample complexity bounds by using Rademacher complexity. For example, given a sample $\sample$ of size $N$, for any $f \in \mathcal{F}$, we can bound the difference between the average value of $f$ over $\sample$ and the expected value of $f$ with respect to $\mathcal{D}$. Formally, with probability at least $1-\delta$, for all $f \in \mathcal{F}$, \begin{equation}\label{eq:rad_bound}\underset{x \sim \mathcal{D}}{\mathbb{E}}\left[f(x)\right] - \frac{1}{N} \sum_{i = 1}^N f\left(x_i\right) \leq 2\mathcal{R}_N(\mathcal{F}) + c \sqrt{\frac{2\ln (2/\delta)}{N}}.\end{equation}

Moreover, for a sample $\sample$, suppose $\hat{h} \in \mathcal{H}$ is the hypothesis that minimizes average loss over $\sample$ and $h^*$ is the hypothesis that minimizes expected loss with respect to the distribution $\mathcal{D}$. Then recalling that $\mathcal{F} \defeq \ell \circ \mathcal{H}$, we can guarantee that with probability at least $1-\delta$, \[\underset{x \sim \mathcal{D}}{\mathbb{E}}\left[\ell\left(\hat{h}, x\right)\right] - \underset{x \sim \mathcal{D}}{\mathbb{E}}\left[\ell\left(h^*, x\right)\right] \leq 2\widehat{\mathcal{R}}_\sample(\mathcal{F}) + 5c \sqrt{\frac{2\ln (8/\delta)}{N}}.\]

\subsubsection{Pseudo-Dimension}

The pseudo-dimension of a class of functions $\mathcal{F}$ is another means of analyzing the complexity of $\mathcal{F}$, and thereby deriving useful sample complexity bounds. To define pseudo-dimension, let $\sample = \left\{x_1, \dots, x_N\right\}$ be a sample drawn from $\mathcal{D}$ and let $\left(z^1, \dots, z^N\right) \in \R^N$ be a set of \emph{targets}. We say that $\left(z^1, \dots, z^N\right)$ \emph{witnesses} the shattering of $\sample$ by $\mathcal{F}$ if for all $T \subseteq \sample$, there exists some function $f_T \in \mathcal{F}$ such that for all $x_i \in T$, $f_T(x_i) \leq z^i$ and for all $x_i \not\in T$, $f_T(x_i) > z^i$. If there exists some $\vec{r}$ that witnesses the shattering of $\sample$ by $\mathcal{F}$, then we say that $\sample$ is \emph{shatterable} by $\mathcal{F}$. Finally, the pseudo-dimension $d_{\mathcal{F}}$ of $\mathcal{F}$ is the size of the largest set that is shatterable by $\mathcal{F}$.\footnote{Note that the pseudo-dimension of $\mathcal{F}$ is simply the VC dimension of the set of ``below-the-graph'' indicator functions $B_{\mathcal{F}} = \left\{ (x, z) \mapsto \text{sgn}\left(f(x) - z\right) \ | \ f \in \mathcal{F}\right\}$ \cite{Anthony09:Neural}.}

By bounding the pseudo-dimension of a class of functions, we can then bound the number of samples $N$ required to ensure that the average value of a function over the sample is close to its expected value with respect to $\mathcal{D}$.

\begin{theorem}[e.g. \cite{Mohri12:Foundations}]\label{thm:pseudo}
Let $\mathcal{F}$ be a family of real-valued functions such that $Pdim(\mathcal{F}) = d_{\mathcal{F}}$ and that every $f \in \mathcal{F}$ has a range bounded by $c$. Then, for any $\delta > 0$, with probability at least $1-\delta$ over the choice of a sample $\sample$ of size $N$, the following inequality holds for all $f \in \mathcal{F}$:

\[\underset{x \sim \mathcal{D}}{\mathbb{E}}[f(x)] \leq \frac{1}{N}\sum_{x \in \sample}f(x) + c \sqrt{\frac{2d\log \frac{eN}{d_{\mathcal{F}}}}{N}} + c \sqrt{\frac{\log \frac{1}{\delta}}{2N}}.\]
\end{theorem}

As will be exemplified in the present paper, it can be more natural to derive sample complexity results via either pseudo-dimension or Rademacher complexity depending on the structure of the function class. Although the two measurements seem far removed, they can be connected both conceptually and mathematically through the learning theoretic concept of \emph{covering numbers}. In particular, $\widehat{\mathcal{R}}_\sample (\mathcal{F}) = \tilde{O}\left(\sqrt{d_{\mathcal{F}}}/N\right)$. For completeness, we describe this connection in more detail in Appendix~\ref{app:techniques_app}.

%% file: AMA.tex
\section{The Sample Complexity of AMA Revenue Maximization}\label{sec:AMA_sample}

We begin with the most general family in the CA hierarchy, affine maximizer auctions. In Section~\ref{sec:AMAupper}, we bound the Rademacher complexity of the class of $n$-bidder, $m$-item AMA revenue functions $\mathcal{F}$. We set our loss function to be $\ell\left(rev_A, \vec{v}\right) =-rev_A\left(\vec{v}\right)$ for any $rev_A \in \mathcal{F}$ and any vector of bidder valuations $\vec{v}$. Therefore, the empirical loss minimizer is the revenue function of the auction with the \emph{maximum} revenue over the sample $\sample$ and the revenue function with the smallest expected loss corresponds to the best auction with respect to the underlying distribution. By bounding the sample complexity of uniform convergence $N$ over the class of AMAs, we may guarantee that if $\sample = \left\{\vec{v}^1, \dots, \vec{v}^{N'}\right\}$ is a set of samples drawn from the underlying distribution $\mathcal{D}$ of size at least $N$, then with probability at least $1-\delta$, for any AMA $A$, $\left|\frac{1}{{N'}}\sum_{i = 1}^{N'} rev_A\left(\vec{v}^i\right) - \mathbb{E}_{\vec{v} \sim \mathcal{D}}\left[rev_A\left(\vec{v}\right)\right]\right| < \epsilon$.

\subsection{Upper Bounds on Sample Complexity for AMAs, VVCAs, and $\lambda$-Auctions}\label{sec:AMAupper}

The AMA revenue function, defined in Section~\ref{sec:prelim}, can be summarized as a multi-stage optimization procedure: determine the weighted-optimal allocation and then compute
the $n$ different payments, each of which requires a separate optimization procedure. In this way, the class of AMA revenue functions is unlike the well-understood, commonly found function classes in machine learning contexts. Luckily, we are able to decompose the revenue functions into small components, each of which is easier to analyze on its own, and then combine our results to prove the following theorem about this class of revenue functions as a whole.

\begin{theorem}\label{thm:AMAupper}
Let $\mathcal{F}$ be the set of $n$-bidder, $m$-item AMA revenue functions $rev_A$ such that $A = \left(w_1, \dots, w_n,\lambda_1,\dots, \lambda_{(n+1)^m}\right),H_{\underline{w}}\leq \left|w_i\right|\leq H_{\overline{w}}, \left|\lambda_i\right| \leq H_\lambda$. Then
\[\mathcal{R}_N(\mathcal{F}) = O\left( \frac{n^{m+2} \left(H_{\overline{w}}H_v+ H_\lambda\right)}{H_{\underline{w}}}\sqrt{\frac{m\log n}{N}}\left(\frac{n\hat{H}_v\left(nH_{\overline{w}} + H_\lambda\right)}{H_{\underline{w}}} + \sqrt{n^m \log N}\right)\right),\] where $\hat{H}_v = \max\left\{H_v, 1\right\}$.
\end{theorem}
\begin{proof}
First, we describe how we split each revenue function into smaller, easier to analyze atoms, which together allow us to bound the Rademacher complexity of the class of AMA revenue functions. To this end, it is well-known (e.g. \cite{Mohri12:Foundations}) that if every function $f$ in a class $\mathcal{F}$ can be written as the summation of two functions $g$ and $h$ from classes $\mathcal{G}$ and $\mathcal{H}$, respectively, then $\mathcal{R}_N(\mathcal{F}) \leq \mathcal{R}_N(\mathcal{G}) + \mathcal{R}_N(\mathcal{H})$. Therefore, we split each revenue function into $n+1$ components such that the sum of these components equals the revenue function.

With this objective in mind, let $\vec{o}^*_A(\vec{v}) = \text{argmax}_{\vec{o}_i \in \mathcal{O}} \left\{\sum_{j = 1}^n w_jv_j\left(o_{i,j}\right) + \lambda_i\right\}$ and $\phi_{A,-j}(\vec{v}) = \max_{\vec{o}_i \in \mathcal{O}} \left\{\sum_{\ell \not= j} w_\ell v_\ell\left(o_{i,\ell}\right) + \lambda_i\right\}.$ Then we can write
\[rev_A(\vec{v}) = \sum_{j= 1}^n \frac{1}{w_j} \phi_{A, -j}(\vec{v}) - \sum_{i = 1}^{(n+1)^m} \left(\sum_{j = 1}^n \frac{1}{w_j} \sum_{\ell\not= j} w_\ell v_\ell(o_{i, \ell}) + \lambda_i\right) \mathbbm{1}_{\vec{o}_i = \vec{o}_A^*(\vec{v})}.\]

We can now split $rev_A$ into $n+1$ simpler functions: $rev_{A,j}(\vec{v}) = \frac{1}{w_j} \phi_{A, -j}(\vec{v})$ for $j \in [n]$ and
\[rev_{A,n+1}(\vec{v}) = - \sum_{i = 1}^{(n+1)^m} \left(\sum_{j = 1}^n \frac{1}{w_j} \sum_{\ell\not= j} w_\ell v_\ell\left(o_{i, \ell}\right) + \lambda_i\right) \mathbbm{1}_{\vec{o}_i = \vec{o}_A^*(\vec{v})},\]
 so $rev_A(\vec{v}) = \sum_{j = 1}^{n+1} rev_{A,j}(\vec{v}).$ Intuitively, for $j \in [n]$, $rev_{A,j}$ is a weighted version of what the social welfare would be if Bidder $j$ had not participated in the auction, whereas $rev_{A,n+1}(\vec{v})$ measures the amount of revenue subtracted to ensure that the resulting auction is strategy-proof.

 As to be expected, bounding the Rademacher complexity of each smaller class of functions $\mathcal{L}_j = \left\{rev_{A,j} \ | \ \left(w_1, \dots, w_n,\lambda_1,\dots, \lambda_{(n+1)^m}\right),H_{\underline{w}}\leq \left|w_i\right|\leq H_{\overline{w}}, \left|\lambda_i\right| \leq H_\lambda\right\}$ for $j \in [n+1]$ is simpler than bounding the Rademacher complexity the class of revenue functions itself and, if $\mathcal{F}$ is the set of all $n$-bidder, $m$-item AMA revenue functions,
 then $\mathcal{R}_N(\mathcal{F}) \leq \sum_{j = 1}^{n+1} \mathcal{R}_N(\mathcal{L}_j).$ In Lemma~\ref{lemma:Lj} and Lemma~\ref{lemma:Ln+1} of Section~\ref{app:AMAupper}, we obtain bounds on $\mathcal{R}_N(\mathcal{L}_j)$ for $j \in [n+1]$ which lead us to our bound on $\mathcal{R}_N(\mathcal{F})$. \hfill \text{ }
\end{proof}

Using these tools, we are now ready to derive the proof of the main sample complexity result stated in Theorem~\ref{thm:main_sample_AMA} in the introduction.

\begin{namedtheorem}[\ref{thm:main_sample_AMA}]
The sample complexity of uniform convergence over the classes of $n$-bidder, $m$-item AMAs, VVCAs, and $\lambda$-Auctions is
\[N = \widetilde{O} \left(\left[\frac{U}{\epsilon}n^m\sqrt{m} \left( U+ n^{m/2}\right)\right]^2\right).\]
Moreover, for $\lambda$-Auctions, $N = \Omega\left(n^m\right)$ and for VVCAs, $N= \Omega\left(2^m\right)$.
\end{namedtheorem}

%

\begin{proof}
For the upper bound, we bound the right-hand-side of Equation~\ref{eq:rad_bound} by $\epsilon$, using the bound on $\mathcal{R}_N(\mathcal{F})$ from Theorem~\ref{thm:AMAupper}, and solve for $N$, using the well-known inequality $\ln x \leq \alpha x - \ln \alpha - 1$ for all $x,\alpha > 0$. We also use the fact that if $U$ is the maximum revenue achievable by an AMA in the setting at hand, then we may write $U = \frac{n}{H_{\underline{w}}}\left(nH_{\overline{w}}H_v + H_\lambda\right)$. The lower bounds follow from Theorem~\ref{thm:AMAlower} and \ref{thm:VVCA_high_low}.
\end{proof}

\subsection{Lower Bound on Sample Complexity for $\lambda$-Auctions}\label{sec:AMAlower}
In this section, we show that it is not possible to learn over the set of $\lambda$-auction revenue functions under an arbitrary distribution with subexponential sample complexity. Since $\lambda$-auctions are a subset of AMAs, this lower bound applies to AMAs as well. In particular, we prove Theorem~\ref{thm:AMAlower}, which states that no algorithm can learn over the class of $n$-bidder, $m$-item $\lambda$-auction revenue functions with sample complexity $o\left(n^m\right)$. This holds even when the bidders' valuation functions are additive.

To prove Theorem~\ref{thm:AMAlower}, we construct a set $V$ of $n$-bidder, $m$-item valuation functions taking values in $\{0,1\}$ where, under each valuation function, each bidder is interested in a specific subset of items, and these subsets are all pairwise disjoint. Moreover, $|V| = n^m - n$. The high level idea is to show that for any subset $H$ of $V$, there exists a $\lambda$-auction that has high revenue over valuation functions in $H$, but low revenue on the valuation functions in $V \setminus H$. Theorem~\ref{thm:AMA_high_low} describes $V$ in more detail. Now suppose that the distribution over the bidders' valuation functions is the uniform distribution over $V$. This means that if a learning algorithm's input samples consist of only a small subset of $V$, then we cannot guarantee that any output revenue function will achieve average revenue over the sample which is close to its expected revenue over the distribution, as we require. This immediately implies hardness for learning over the uniform distribution on $V$. See Theorem~\ref{thm:AMAlower} for the formal proof.

We now present Theorem~\ref{thm:AMA_high_low}, wherein we describe the set $V$ of valuation functions which we will use to prove Theorem~\ref{thm:AMAlower}.

\begin{theorem}\label{thm:AMA_high_low}
For any $n,m \geq 2$ and any $\gamma \in (0,1)$, there exists a set of $N = n^m-n$ $n$-bidder, $m$-item additive valuation functions $V = \left\{\vec{v}^1, \dots, \vec{v}^N\right\}$ such that for any $H \subseteq V$, there exists a $\lambda$-auction $A_H$ with revenue 0 on $\vec{v}^i$ if $\vec{v}^i \not\in H$ and revenue at least $2-2\gamma$ on $\vec{v}^i$ otherwise.
\end{theorem}

\begin{proof}
We define the set $V = \left\{\vec{v}^1, \dots, \vec{v}^N\right\}$ of $n$-bidder, $m$-item additive valuation functions, where $\vec{v}^j = \left(v_1^j(\{1\}), \dots, v_1^j(\{m\}), \dots, v_n^j(\{1\}) \dots, v_n^j(\{m\})\right)$, with $N = n^m - n$. Recall that every allocation vector $\vec{o}_j$ is written as $\left(o_{j,1},\dots,  o_{j,n}\right)$ where $o_{j,1},\dots,  o_{j,n}$ are disjoint subsets of the $m$ items being auctioned. First, let $\hat{o}_j$ be the allocation where Bidder $j$ receives all $m$ items. Next, let $\tilde{o}_1, \dots, \tilde{o}_N$ be a fixed ordering of the $n^m-n$ allocations where all $m$ goods are allocated except $\left\{\hat{o}_1, \dots, \hat{o}_n\right\}$. Let the bundles allocated to the $n$ bidders in $\tilde{o}_\ell$ be $\left(\tilde{o}_{\ell,1}, \dots, \tilde{o}_{\ell,n}\right)$ and let $N_\ell$ be the set of bidders who are allocated some item in allocation $\tilde{o}_{\ell}$. In other words, $N_{\ell} = \left\{j \ | \ \tilde{o}_{\ell, j} \not= \emptyset\right\}$. For a sanity check, notice that $\bigcup_{i = 1}^n \tilde{o}_{\ell,i}$ is the grand bundle.

We will now define the valuation vectors $\left\{\vec{v}^1, \dots, \vec{v}^N\right\}$ in terms of this set of special allocations $\left\{\tilde{o}_1, \dots, \tilde{o}_N\right\}$. Specifically, we define $\vec{v}^\ell$ for $\ell \in [N]$ as follows.

If $i \not\in N_\ell$ $\left( \text{i.e. }\tilde{o}_{\ell,j} = \emptyset\right)$, set $v_i^\ell(\{j\}) = 0$ for all $j \in [m]$. Otherwise, set \[v_i^\ell(\{j\}) = \begin{cases} 0 &\text{if } j \not\in \tilde{o}_{\ell,i}\\
1 &\text{if } j \in \tilde{o}_{\ell,i}
\end{cases}.\]

We proceed to prove that for any subset $H \subseteq V$, there exists a $\lambda$-auction with 0 revenue on all valuation functions in $V \setminus H$ and at least $2-2\gamma$ revenue on all valuation functions in $H$. To define this $\lambda$-auction, we set the $\lambda$ terms such that \[\lambda\left(\vec{o_j}\right) = \begin{cases}
0 &\text{if } \vec{o}_j = \tilde{o}_\ell \text{ for some } \vec{v}^\ell \in H\\
1 - \gamma &\text{otherwise}\end{cases}.\]

\begin{lemma}\label{lemma:n_bid_high}
If $\vec{v}^\ell \in H$, then the revenue on $\vec{v}^\ell$ is at least $2-2\gamma$.
\end{lemma}

\begin{proof}[Proof of Lemma~\ref{lemma:n_bid_high}]
First, note that $\sum_{i = 1}^n v_i^\ell\left(\tilde{o}_{\ell,i}\right) + \lambda\left(\tilde{o}_\ell\right) = m$, and for all allocations $\vec{o}_j \not = \tilde{o}_\ell$, $\sum_{i = 1}^n v_i^\ell\left(o_{j,i}\right) + \lambda\left(\vec{o}_j\right) \leq m-1+1-\gamma < m$. Therefore, the $\lambda$-auction allocation is $\tilde{o}_\ell$.

In order to analyze the revenue of this $\lambda$-auction, we must understand the payments of each bidder, which means that we must investigate what the outcome of this $\lambda$-auction would be without any one bidder's participation. To this end, suppose $i \in N_\ell$, so Bidder $i$ is allocated some item in $\tilde{o}_{\ell},$ i.e. $\tilde{o}_{\ell,i} \not= \emptyset$. Then $\sum_{j \not= i} v_j^\ell\left(\tilde{o}_{\ell,j}\right) + \lambda\left(\tilde{o}_\ell\right) = m-\left|\tilde{o}_{\ell,i}\right|$ because Bidder $i$'s valuation for the bundle $\tilde{o}_{\ell,i}$ is exactly $\left|\tilde{o}_{\ell,i}\right|$.

By construction, no bidder receives all $m$ items in $\tilde{o}_\ell$, so we know that there exists some $i' \in N_\ell, i' \not = i$. With this fact in mind, let $\vec{o}_{-i}$ be the allocation where all bidders in $N_\ell$ are allocated the same items as they are in $\tilde{o}_\ell$ and Bidder $i$ receives the empty set. This is one possible allocation of the $\lambda$-auction without Bidder $i$'s participation, and therefore the social welfare of the other bidders will be at least as high under this allocation as it would be in the true allocation of the $\lambda$-auction without Bidder $i$'s participation. By construction, $\lambda\left(\vec{o}_{-i}\right) = 1 - \gamma$. Therefore, $\sum_{\ell \not= i} v_j^\ell \left(o_{-i,j}\right) + \lambda\left(\vec{o}_{-i}\right) = m-\left|\tilde{o}_{\ell,i}\right| + 1 - \gamma$ which means that Bidder $i$ must pay at least $\left(m-\left|\tilde{o}_{\ell,i}\right| + 1 - \gamma\right) - \left(m-\left|\tilde{o}_{\ell,i}\right|\right) = 1 - \gamma.$ We know that $|N_\ell|\geq 2$, i.e. there are at least 2 bidders who receive a non-empty bundle and therefore must pay at least $1 - \gamma$, so the revenue of this $\lambda$-auction is at least $2-2\gamma$.
\hfill \text{ }
\end{proof}

\begin{lemma}\label{lemma:n_bid_low} If $\vec{v}^\ell \not\in H$, then the revenue on $\vec{v}^\ell$ is 0.
\end{lemma}

\begin{proof}[Proof of Lemma~\ref{lemma:n_bid_low}]
First, note that $\sum_{i = 1}^n v_i^\ell\left(\tilde{o}_{\ell,i}\right) + \lambda\left(\tilde{o}_\ell\right) = m + 1-\gamma$, and for all allocations $\vec{o}_j \not= \tilde{o}_\ell$, $\sum_{i = 1}^n v_i^\ell\left(\vec{o}_{j,i}\right) + \lambda\left(\vec{o}_j\right) \leq m-1 + 1-\gamma < m$, so the $\lambda$-auction allocation is $\tilde{o}_\ell$. Now, suppose $i \in N_\ell$. Then $\sum_{j \not= i} v_j^\ell\left(\tilde{o}_{\ell,j}\right) + \lambda\left(\tilde{o}_\ell\right) = m - \left|\tilde{o}_{\ell,i}\right| + 1 - \gamma.$ Since Bidder $i$ is the only bidder with nonzero valuations for the items in $\tilde{o}_{\ell,i}$ under $\vec{v}^\ell$, any allocation $\vec{o}_{-i}$ without his participation will have social welfare at most $\sum_{j \not= i} v_j^\ell\left(o_{-i,j}\right) + \lambda\left(\vec{o}_{-i}\right) \leq m - \left|\tilde{o}_{\ell,i}\right| + 1 - \gamma.$ Therefore, Bidder $i$ pays nothing.

Of course, for any Bidder $i \not\in N_\ell$, her presence in the auction makes no difference on the resulting allocation because her valuation function under $\vec{v}^\ell$ is 0 on all items, so she pays nothing as well. Therefore, the revenue on $\vec{v}^\ell$ is 0.
\hfill \text{ }
\end{proof}

Putting Lemmas~\ref{lemma:n_bid_high} and~\ref{lemma:n_bid_low} together, we have the desired result.
\hfill \text{ }
\end{proof}

We now use Theorem~\ref{thm:AMA_high_low} to prove Theorem~\ref{thm:AMAlower}.

\begin{theorem}\label{thm:AMAlower}
Let $\mathcal{ALG}$ be an arbitrary learning algorithm that uses only a polynomial number of training samples drawn i.i.d. from the underlying distribution and produces a $\lambda$-auction. For any $\epsilon \in (0,1)$, there exists a distribution $\mathcal{D}$ and a $\lambda$-auction $A^*$ such that, with probability 1 (over the draw of the set of training samples $\mathcal{S}$), \[\frac{1}{|\sample|}\sum_{\vec{v} \in \sample} rev_{A^*} \left(\vec{v}\right) - \underset{\vec{v}\sim \mathcal{D}}{\mathbbm{E}}\left[ rev_{A^*}\left(\vec{v}\right)\right] > \epsilon.\]
\end{theorem}

\begin{proof}
Let $\gamma = 1 - \epsilon$ and let $V$ be the set of valuation functions proven to exist in Theorem~\ref{thm:AMA_high_low} corresponding to $\gamma$ (i.e. for any $H \subseteq V$, there exists a $\lambda$-auction $A_H$ with revenue 0 on $\vec{v}$ if $\vec{v} \in H$ and revenue at least $2-2\gamma$ on $\vec{v}$ otherwise). Let $\mathcal{D}$ be the uniform distribution on $V$.

Suppose that $\mathcal{ALG}$ uses a set $\mathcal{S}$ of $\ell \leq n^c$ samples, where $c$ is a constant. Of course, $\sample \subseteq V$, so let $A^*$ be the $\lambda$-auction with 0 revenue on every valuation function not in the sample and revenue at least $2 - 2\gamma$ on every valuation function in the sample. We know that $A^*$ exists due to Theorem~\ref{thm:AMA_high_low}.

Notice that the average empirical revenue of $A^*$ on $\sample$ is at least $2 - 2\gamma$. Meanwhile, the probability, on a random draw $\vec{v} \sim \mathcal{D}$ that $rev_{A^*}\left(\vec{v}\right)$ is 0 is exactly the probability that $\vec{v} \not\in \sample$. Given that the set of training examples has measure $\frac{n^c}{n^m - n}<\frac{1}{2},$ we have that \begin{align*}
\frac{1}{|\sample|}\sum_{\vec{v} \in \sample} rev_{A^*} \left(\vec{v}\right) - \underset{\vec{v}\sim \mathcal{D}}{\mathbbm{E}}\left[ rev_{A^*}\left(\vec{v}\right)\right] &\geq 2-2\gamma - (2-2\gamma)\underset{\vec{v}\sim\mathcal{D}}{\mathbbm{P}}\left[\vec{v} \in \sample\right]\\
&> 2-2\gamma - (1-\gamma)\\
&= 1-\gamma\\
&= \epsilon,
\end{align*}
as desired.
\end{proof}

\subsection{Lower Bound on Sample Complexity for VVCAs}\label{app:VVCAlower}

In this section, we prove that it is not possible to learn over the set of VVCA revenue function under and arbitrary distribution with subexponential sample complexity. In particular, we prove that no algorithm can learn over the class of $n$-bidder, $m$-item VVCA revenue functions with sample complexity $o\left(2^m\right)$. This holds even when the bidders' valuation functions are additive.

The format of this proof similar to that of Theorem~\ref{thm:AMAlower}. Namely, we construct a set $V$ of $n$-bidder, $m$-item valuation functions such that $|V| = 2^m - 2$. We then show that for any subset $H$ of $V$, there exists a VVCA that has high revenue over valuation functions in $H$, but low revenue on the valuation functions in $V \setminus H$. The set $V$ is described in more detail in Theorem~\ref{thm:VVCA_high_low}. As described in Theorem~\ref{thm:AMAlower}, this immediately implies hardness for learning over the uniform distribution on $V$. Given the parallel proof structure, we present Theorem~\ref{thm:VVCA_high_low} and refer the reader to Theorem~\ref{thm:AMAlower} to see how it implies hardness for learning.

\begin{theorem}\label{thm:VVCA_high_low}
For any $m \geq 2$ and any $\gamma \in (0,1)$, there exists a set of $N = 2^m-2$ 2-bidder additive valuation functions $V = \{\vec{v}^1, \dots, \vec{v}^N\}$ such that for any $H \subseteq V$, there exists a VVCA with revenue 0 on $\vec{v}^i$ if $\vec{v}^i \in V$ and revenue $1-\gamma$ on $\vec{v}^i$ if $\vec{v}^i \not\in V$.
\end{theorem}

\begin{proof}
We define the set $V = \{\vec{v}^1, \dots, \vec{v}^N\}$ of 2-bidder valuation functions, where \newline $\vec{v}^j = (v_1^j(\{1\}), \dots, v_1^j(\{m\}), v_2^j(\{1\}) \dots, v_2^j(\{m\}))$, with $N = 2^m - 2$. Recall that every allocation vector $\vec{o}_j$ can be written as $(o_{j,1}, o_{j,2})$ where $o_{j,1}$ and $o_{j,2}$ are disjoint subsets of the $m$ items being auctioned. In order to define the valuation functions in $V$, we define $\tilde{b}_1, \dots, \tilde{b}_N$ to be a arbitrary, fixed ordering of all subsets of $[m]$ except the empty set and the grand bundle. In other words, $\tilde{b}_1, \dots, \tilde{b}_N$ is an ordering of $2^{[m]}\setminus \{\emptyset, [m]\}$. We will define each valuation function in $V$ in terms of this ordering. In particular, let $\tilde{o}_\ell = (\tilde{b}_\ell^c, \tilde{b}_\ell)$ be the allocation where Bidder 1 receives $\tilde{b}_\ell^c$ and Bidder 2 receives $\tilde{b}_\ell$. Finally, let $\vec{v}^\ell$ for $\ell \in [N]$ be defined as follows.

\[v_1^\ell(\{i\}) = \begin{cases} 1 &\text{if } i \in \tilde{b}_\ell^c\\
0 &\text{otherwise}
\end{cases}\] and
\[v_2^\ell(\{i\}) = \begin{cases} 1 &\text{if } i \in \tilde{b}_\ell\\
0 &\text{otherwise}
\end{cases}.\]

Clearly, if $w_1=w_2=1$ and $\lambda_1(\vec{o}) = \lambda_2(\vec{o}) = 0$ for all $\vec{o} \in \mathcal{O}$, then the VVCA allocation on any $\vec{v}^\ell \in S$ is the one in which Bidder 2 receives $\tilde{b}_\ell^c$ and Bidder 1 receives $\tilde{b}_\ell$. This has a social welfare of $m$, whereas any other allocation has a social welfare at most $m-1$.

We claim that for any $H \subseteq V$, there exists a VVCA with revenue 0 on $\vec{v}^i$ if $\vec{v}^i \in H$ and revenue $1-\gamma$ on $\vec{v}^i$ if $\vec{v}^i \not\in H$. The VVCA has bidder weights $w_1 = w_2 = 1$, and for all $\vec{v}^\ell \in H$, we set $\lambda_1(\tilde{o}_\ell) = c_{1,\tilde{b}_\ell^c} = c_{2, \tilde{b}_\ell} = \lambda_2(\tilde{o}_\ell) = 0$. Otherwise, we set $\lambda_i(\vec{o}) = (1 - \gamma)/2$ for each $i \in \{1,2\}$.

\begin{lemma}\label{lem:VVCA_high} If $\vec{v}^\ell \in H$, then the revenue on $\vec{v}^\ell$ is $1-\gamma$.
\end{lemma}

\begin{proof}[Proof of Lemma~\ref{lem:VVCA_high}]
First, note that $v_1(\tilde{b}_\ell^c) + v_2(\tilde{b}_\ell) +\lambda_1(\tilde{o}_\ell) + \lambda_2(\tilde{o}_\ell) = m$, and for all allocations $\vec{o}_j \not = \tilde{o}_\ell$, $v_1(o_{j,1}) + v_2(o_{j,2}) + \lambda_1(\vec{o}_j) + \lambda_2(\vec{o}_j) \leq m-1 + 1-\gamma$. Therefore, the VVCA allocation is $\tilde{o}_\ell$. However, this is neither Bidder 1 nor Bidder 2's favorite weighted allocation, since $v_1(\tilde{b}_\ell^c) + \lambda_1(\tilde{o}_\ell) = |\tilde{b}_\ell^c| < v_1([m]) + c_{1,[m]} = |\tilde{b}_\ell^c| + (1-\gamma) /2$ and $v_2(\tilde{b}_\ell) + \lambda_2(\tilde{o}_\ell) = |\tilde{b}_\ell|< v_2([m]) + c_{2,[m]} = |\tilde{b}_\ell| + (1-\gamma) /2$. This follows from the fact that $\tilde{b}_\ell \not = [m]$ and $\tilde{b}_\ell^c \not = [m]$ for all $\ell \in [N]$, it must be that $\lambda_1([m]) = \lambda_2([m]) = (1 - \gamma)/2.$

Since $|\tilde{b}_\ell^c|$ and $|\tilde{b}_\ell|$ are Bidder 1 and 2's highest valuations for any allocation, respectively, and because $(1-\gamma)/2$ is the highest value of any $\lambda$ term, $v_1([m]) + c_{1,[m]}$ and $v_2([m]) + c_{2,[m]}$ are the maximum weighted valuation that either bidder has for any allocation under this VVCA. Therefore, the revenue of this VVCA on $\vec{v}_\ell$ is $|\tilde{b}_\ell| + |\tilde{b}_\ell^c| + 1-\gamma - |\tilde{b}_\ell| - |\tilde{b}_\ell^c| = 1-\gamma$.
\hfill \text{ }
\end{proof}

\begin{lemma}\label{lem:VVCA_low} If $\vec{v}^\ell \not \in H$, then the revenue on that valuation function pair is 0.
\end{lemma}

\begin{proof}[Proof of Lemma~\ref{lem:VVCA_low}]
First, note that $v_1(\tilde{b}_\ell^c) + v_2(\tilde{b}_\ell) +\lambda_1(\tilde{o}_\ell) + \lambda_2(\tilde{o}_\ell) = m+1-\gamma$, and for all allocations $\vec{o}_j \not= \tilde{o}_\ell$, $v_1(o_{j,1}) + v_2(o_{j,2}) + \lambda_1(\vec{o}_j) + \lambda_2(\vec{o}_j) \leq m-1 + 1-\gamma < m + 1-\gamma$, so the AMA allocation is $\tilde{o}_\ell$. Moreover, $v_1(\tilde{b}_\ell^c) + \lambda_1(\tilde{o}_\ell) = |\tilde{b}_\ell^c| + (1-\gamma)/2 \geq v_1(o_{j,1}) + \lambda_1(\vec{o}_j)$ and  $v_2(\tilde{b}_\ell) + \lambda_2(\tilde{o}_\ell) = |\tilde{b}_\ell| + (1-\gamma)/2 \geq v_2(o_{j,2}) + \lambda_2(\vec{o}_j)$ for all allocations $\vec{o}_j \in \mathcal{O}$. Therefore, both bidders receive one of their favorite weighted allocations, so the revenue is 0.
\hfill \text{ }
\end{proof}
\hfill \text{ }
\end{proof}

%% file: MBA.tex
\section{Sample Complexity of MBA Revenue Maximization}\label{sec:MBA_sample}

Fortunately, these negative sample complexity results are not the end of the story. We do achieve polynomial sample complexity upper bounds for the important classes of mixed bundling auctions (MBAs) and mixed bundling auctions with reserve prices (MBARPs). We derive these sample complexity bounds by analyzing the pseudo-dimensions of these classes of auctions. In this section, we present our results in increasing complexity, beginning with the class of $n$-bidder, $m$-item MBAs, which we show has a pseudo-dimension of 2. We build on the proof of this result to show that the class of $n$-bidder, $m$-item MBARPs has a pseudo-dimension of $O\left(m^3 \log n\right)$.

We note that when we analyze the class of MBARPs, we assume additive reserve prices, rather than bundle reserve prices. In other words, each item has its own reserve price, and the reserve price of a bundle is the sum of its components' reserve prices, as opposed to each bundle having its own reserve price. We have good reason to make this restriction; in Section~\ref{sec:bundle_lower}, we prove that an exponential number of samples are required to learn over the class of MBARPs with bundle reserve prices.

 Before we prove our sample complexity results, we fix some notation. For any $c$-MBA, let $rev_c\left(\vec{v}\right)$ be its revenue on $\vec{v}$, which is determined in the exact same way as the general AMA revenue function with the $\lambda$ terms set as described in Section~\ref{sec:prelim}.

We will use the following result regarding the structure of $rev_c\left(\vec{v}\right)$ in order to derive our pseudo-dimension results.

\begin{lemma}\label{lem:rev_struct}
There exists $c^* \in [0,\infty)$ such that $rev_{\vec{v}}(c)$ is non-decreasing on the interval $[0,c^*]$ and non-increasing on the interval $(c^*, \infty)$.
\end{lemma}

\begin{proof}
We will show that $rev_{\vec{v}}$ can be decomposed into simple components, each of which can be easily analyzed on its own, and by combining these analyses, we prove the lemma statement. To this end, recall that under the VCG mechanism, each winning bidder pays her bid minus a ``rebate'' equal to the increase in
welfare attributable to her presence in the auction. In a $c$-MBA, each winning bidder pays the boosted version of this amount. In other words, suppose $\vec{o}^*$ is the resulting allocation of a certain $c$-MBA $A$ and $\vec{o}_{-i}$ is the boosted social-welfare maximizing allocation without Bidder $i$'s participation. More explicitly, $\vec{o}^* = \max_{\vec{o}_j} \left\{\sum_{i = 1}^n v_i\left(o_{j,i}\right) + \lambda\left(\vec{o}_j\right)\right\}$ and $\vec{o}_{-i} = \max_{\vec{o}_j} \left\{\sum_{k \not=i} v_k\left(o_{j,k}\right) + \lambda\left(\vec{o}_j\right)\right\}$, where $\lambda\left(\vec{o}_j\right)$ is set according to the MBA allocation boosting rule for all $\vec{o}_j$. Then Bidder $i$ pays \[p_{i, \vec{v}}\left(c\right) = v_i\left(o^*_i\right) - \left[\sum_{j = 1}^n v_j\left(o^*_j\right) + \lambda\left(\vec{o}^*\right) - \left(\sum_{j \not= i} v_j\left(o_{-i,j}\right) + \lambda\left(\vec{o}_{-i}\right)\right) \right],\] where $c$ is the parameter of the $c$-MBA, factored into the $\lambda$ terms. This means that \[rev_{\vec{v}}(c) = \sum_{i = 1}^n p_{i, \vec{v}}\left(c\right) = (1 - n)\sum_{i = 1}^n v_i\left(o_i^*\right) - n\lambda\left(\vec{o}^*\right) + \sum_{i = 1}^n \sum_{j \not= i} v_j\left(o_{-i,j}\right) + \lambda\left(\vec{o}_{-i}\right).\]

The revenue function can be split into $n+1$ functions:

\[f_{i,\vec{v}}(c) = \sum_{j \not= i} v_j\left(o_{-i,j}\right) + \lambda\left(\vec{o}_{-i}\right) \text{ for }i \in \{1, \dots, n\}\] and \[g_{\vec{v}}(c) = (1 - n)\sum_{i = 1}^n v_i\left(o_i^*\right) - n\lambda\left(\vec{v}^*\right).\] We claim that $f_{i,\vec{v}}(c)$ is continuous for all $i$, whereas $g_{\vec{v}}(c)$ has at most one discontinuity. This means that $rev_{\vec{v}}(c) = \sum_{i = 1}^n f_{i,\vec{v}}(c) + g_{\vec{v}}(c)$ has at most one discontinuity as well. Moreover, the slope of $\sum_{i = 1}^n f_{i,\vec{v}}(c)$ is between zero and $n$, whereas the slope of $g_{\vec{v}}(c)$ is zero until its discontinuity, and then is $-n$. Therefore, the slope of $rev_{\vec{v}}(c)$ is at least zero before its discontinuity and at most zero after its discontinuity. This is enough to prove the lemma statement.

To see why these properties are true for the functions $f_{i,\vec{v}}(c)$, first let $\vec{o}_{-i}^1$ be the VCG allocation without Bidder $i$'s participation. In other words, $\vec{o}_{-i}^1 = \max_{\vec{o}_j} \left\{\sum_{k \not=i} v_k\left(o_{j,k}\right)\right\}$. If one bidder is allocated the grand bundle in outcome $\vec{o}_{-i}^1$, then this allocation will only be more valuable as $c$ grows, so $\vec{o}_{-i}^1 = \max_{\vec{o}_j} \left\{\sum_{k \not=i} v_k\left(o_{j,k}\right) + \lambda\left(\vec{o}_j\right)\right\}$ for all values of $c$, which means that $f_{i,\vec{v}}(c) = \sum_{j \not= i} v_j\left(o_{-i,j}^1\right) + \lambda\left(\vec{o}_{-i}^1\right) = \sum_{j \not= i} v_j\left(o_{-i,j}^1\right) + c$ for all values of $c$ as well. Clearly, in this case, $f_{i,\vec{v}}(c)$ is increasing and continuous. Otherwise, there exists some value $c_i$ such that \begin{align*}\sum_{j \not= i} v_j\left(o_{-i,j}^1\right) + \lambda\left(\vec{o}_{-i}^1\right) = \sum_{j \not= i} v_j\left(o_{-i,j}^1\right) &\geq \max_{k \not= i} \left\{v_k\left([m]\right)\right\} + c   &\text{ if }c \leq c_i\\
\sum_{j \not= i} v_j\left(o_{-i,j}^1\right) &< \max_{k \not= i} \left\{v_k\left([m]\right)\right\} + c   &\text{ if }c > c_i.\end{align*} This means that $\vec{o}_{-i}^1$ is the allocation of the $c$-MBA without Bidder $i$'s participation for $c \leq c_i$, and the allocation of the $c$-MBA without Bidder $i$'s participation for $c > c_i$ is the one where the highest bidder for the grand bundle (excluding Bidder $i$) wins the grand bundle. Therefore, \[f_{i,\vec{v}}(c) = \begin{cases} \sum_{j \not= i} v_j\left(o_{-i,j}^1\right) & \text{if } c \leq c_i\\ \max_{k \not= i} \left\{v_k\left([m]\right)\right\} + c &\text{if } c > c_i.\end{cases}\] Notice that $\sum_{j \not= i} v_j\left(o_{-i,j}^1\right) = \max_{k \not= i} \left\{v_k\left([m]\right)\right\} + c_i$, so $f_{i,\vec{v}}(c)$ is continuous. Finally, it is clear that the slope of each $f_{i,\vec{v}}(c)$ is between 0 and 1, so the slope of $\sum_{i = 1}^n f_{i,\vec{v}}(c)$ is between 0 and $n$.

Similarly, let $\vec{o}^1$ be the allocation of the VCG mechanism run on $\vec{v}$. Then there exists some $c^*$ such that $\vec{o}^1$ is the allocation of the $c$-MBA for $c \leq c^*$ and the allocation of the $c$-MBA for $c > c_i$ is the one where the highest bidder for the grand bundle wins the grand bundle. More explicitly, \begin{align*}\sum_{i=1}^n v_i\left(o_{i}^1\right) + \lambda\left(\vec{o}^1\right) = \sum_{i=1}^n v_i\left(o_{i}^1\right) &\geq \max\left\{v_k\left([m]\right)\right\} + c   &\text{ if }c \leq c^*\\
\sum_{i=1}^n v_i\left(o_{i}^1\right) &< \max \left\{v_k\left([m]\right)\right\} + c   &\text{ if }c > c^*.\end{align*} Therefore, \[g_{\vec{v}}(c) = \begin{cases} (1-n)\sum_{i=1}^n v_i\left(o_{i}^1\right) & \text{if } c \leq c^*\\ (1-n)\max \left\{v_k\left([m]\right)\right\} - nc &\text{if } c > c^*.\end{cases}\] Therefore, $g_{\vec{v}}(c)$ has at most one discontinuity, which falls at $c^*$. Moreover, the slope of $g_{\vec{v}}(c)$ is 0 for $c < c^*$ and $-n$ for $c > c^*$. As described, these properties of $f_{i,\vec{v}}(c)$ and $g_{\vec{v}}(c)$ are enough to show that the lemma statement holds.

\end{proof}

The form of $rev_{\vec{v}}(c)$ as described in Lemma~\ref{lem:rev_struct} is depicted in Figure~\ref{fig:revGraphn}.
\begin{figure}[h]
  \centering
  \includegraphics[scale=.3]{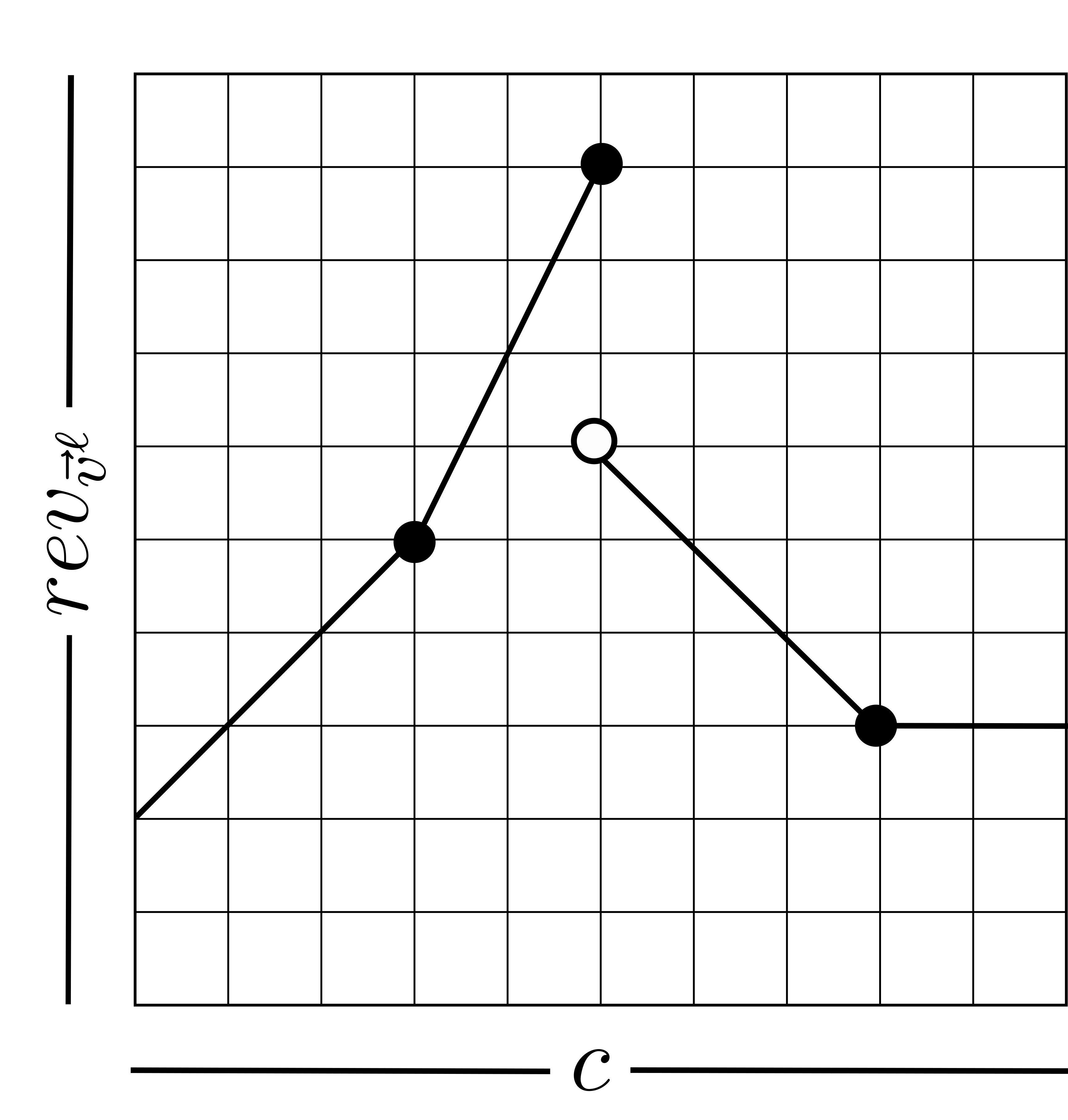}
  \captionof{figure}{Example of $rev_{\vec{v}^\ell}(c)$.}
  \label{fig:revGraphn}
\end{figure}

\begin{theorem}\label{thm:MBAn}

The pseudo-dimension of the class of $n$-bidder, $m$-item MBAs is 2.

\end{theorem}

\begin{proof}
First, we show that the pseudo-dimension of the class of $n$-bidder, $m$-item MBAs is at most 2. Let $\sample = \left\{\vec{v}^1, \dots, \vec{v}^N\right\}$ of size $N$ be a set of $n$-bidder valuation functions that can be shattered by a set $C$ of $2^N$ MBAs. This means that there exist $N$ witnesses $z^1, \dots, z^N$ such that each MBA in $C$ induces a binary labeling of the samples $\vec{v}^j$ of $\sample$ (whether the revenue of the MBA on $\vec{v}^j$ is at least $z^j$ or strictly less than $z^j$). Since $\sample$ is shatterable, we can thus label $\sample$ in every possible way using MBAs in $C$.

Now, fix one sample $\vec{v}^i \in \sample$ and consider $rev_{\vec{v}^i}(c)$. From Lemma~\ref{lem:rev_struct}, we know that there exists $c^*_i \in [0,\infty)$, such that $rev_{\vec{v}^i}(c)$ is non-decreasing on the interval $[0,c^*_i]$ and non-increasing on the interval $(c^*_i, \infty)$. Therefore, there exist two thresholds $t^1_i \in [0,c^*_i]$ and $t^2_i \in (c^*_i, \infty) \cup \{\infty\}$ such that $rev_{\vec{v}^i}(c)$ is below its threshold for $c \in [0,t^1_i)$, above its threshold for $c \in (t^1_i, t^2_i)$, and below its threshold for $c \in (t^2_i, \infty)$. Now, merge these thresholds for all $N$ samples on the real line and consider the interval $(t_1,t_2)$ between two adjacent thresholds. The binary labeling of the samples in $\sample$ on this interval is fixed. In other words, for any sample $\vec{v}^j \in \sample$, $rev_{\vec{v}^j}(c)$ is either at least $z^j$ or strictly less than $z^j$ for all $c \in (t_1,t_2)$. There are at most $2N+1$ intervals between adjacent thresholds, so at most $2N+1$ different binary labelings of $\sample$. Since we assumed $\sample$ is shatterable, it must be that $2^N \leq 2N+1$, so $N \leq 2.$

Finally, we show that the pseudo-dimension of the class of $n$-bidder, $m$-item MBAs is at least 2 by constructing a set $\sample = \left\{\vec{v}^1, \vec{v}^2\right\}$ that can be shattered by the set of MBAs. To construct this sample $\sample$, let \[v_1^1\left(b_i\right) = v_2^1\left(b_i\right) = \begin{cases} 0 &\text{if } \left|b_i\right| < \lfloor m/2 \rfloor\\
3 &\text{if } \lfloor m/2 \rfloor \leq \left|b_i\right| \end{cases} \text{ and } v_1^2\left(b_i\right) = v_2^2\left(b_i\right) = \begin{cases} 0 &\text{if } \left|b_i\right| < \lfloor m/2 \rfloor\\
3 &\text{if } \lfloor m/2 \rfloor \leq \left|b_i\right| < m\\
4 &\text{if } \left|b_i\right| = m. \end{cases}\] Finally, let Bidders 3 through $n$ have all-zero valuations in both $\vec{v}^1$ and $\vec{v}^2$.

Now, let $z^1 = 3$ and $z^2 = 4$. We define four MBAs parameterized by the coefficients $c_1 = 0, c_2 = 1.5, c_3 = 2.5, c_4  = 2.$ It is easy to check that this set of MBAs shatters $\sample$, witnessed by $z^1$ and $z^2$. For example, see Table~\ref{tab:shattering}.

\begin{table}\centering
{\begin{tabular}{lll}
\textbf{$c$ value} & \textbf{Revenue on $\vec{v}^1$} & \textbf{Revenue on $\vec{v}^2$} \\\hline
0                 & $0\leq z^1$                               & $2 \leq z^2$                               \\
1.5                & $3 \leq z^1$                               & $5 > z^2$                               \\
2.5                & $5 > z^1$                               & $4 \leq z^2$                               \\
2                  & $4 > z^1$                               & $6 > z^2$
\end{tabular}}
\caption{Example of a shattered set of size 2\label{tab:shattering}}
\end{table}
\end{proof}

We may now use this result to prove Theorem~\ref{thm:main_sample_MBAn}.

\begin{namedtheorem}[\ref{thm:main_sample_MBAn}] The sample complexity of uniform convergence over the class of $n$-bidder, $m$-item MBAs is \[N = O\left(\left(\frac{U}{\epsilon}\right)^2\left(\log\frac{U}{\epsilon} + \log \frac{1}{\delta}\right)\right).\]
\end{namedtheorem}

\begin{proof}
This follows from Theorem~\ref{thm:pseudo} and Theorem~\ref{thm:MBAn}.
\end{proof}

\subsection{Mixed Bundling Auctions with Reserve Prices (MBARPs)}
In Section~\ref{sec:bundle_lower}, we show that exponentially-many samples are required to learn an optimal setting of the MBA parameter $c$ and reserve prices if we allow for bundle-specific reserve prices. Therefore, we restrict our attention to item-specific reserve prices. In this case, each MBARP is parameterized by $m+1$ values $\left(c, r_1, \dots, r_m\right)$, where $r_i$ is the reserve price for the $i^{th}$ good. For a fixed valuation function vector $\vec{v} = \left(v_1\left(b_1\right), \dots, v_1\left(b_{2^m}\right), \dots, v_n\left(b_1\right), \dots, v_n\left(b_{2^m}\right)\right)$, we can analyze the MBARP revenue function on $\vec{v}$ as a mapping $rev_{\vec{v}}: \R^{m+1} \to \R$, where $rev_{\vec{v}}\left(c, r_1, \dots, r_m\right)$ is the revenue of the MBARP parameterized by $\left(c, r_1, \dots, r_m\right)$ on $\vec{v}$.

\begin{theorem}\label{thm:MBARPupper}
The psuedo-dimension of the class of $n$-bidder, $m$-item MBARPs with item-specific reserve prices is $O\left(m^3 \log n\right)$.
\end{theorem}

\begin{proof}
Let $\sample = \left\{\vec{v}^1, \dots, \vec{v}^N\right\}$ of size $N$ be a set of $n$-bidder valuation function samples that can be shattered by a set $C$ of $2^N$ MBARPs. This means that there exist $N$ witnesses $z^1, \dots, z^N$ such that each MBARP in $C$ induces a binary labeling of the samples $\vec{v}^j$ in $\sample$ (whether the revenue of the MBARP on $\vec{v}^j$ is greater than $z^j$ or at most $z^j$). Since $\sample$ is shatterable, we can thus label $\sample$ in every possible way using MBARPs in $C$.

This proof is similar to the proof of Theorem~\ref{thm:MBAn}, where we split the real line into a set of intervals $\mathcal{I}$ such that for any $I \in \mathcal{I}$, the binary labeling of $\sample$ by the $c$-MBA revenue function was fixed for all $c \in I$. In the case of MBARPs, however, the domain is $\R^{m+1}$, so we cannot split the domain into intervals in the same way. Instead, we show that we can split the domain into cells such that the binary labeling of $\sample$ by the MBARP revenue function is fixed as we range over parameters in a single cell. In this way, we show that $N = O\left(m^3 \log n\right)$.

Now, fix $\vec{v}^t \in \sample$. First,  for each $T \subseteq[m]$, let $\mathcal{O}_T$ be the set of allocations where exactly the elements of $T$ are allocated, and let \[\vec{o}^T = \underset{\vec{o} \in \mathcal{O}_T}{\text{argmax}} \left\{\sum_{i = 1}^n v_i^t\left(o_i\right)\right\}.\] Notice that regardless of the reserve prices, if $T$ comprises of the items allocated in the allocation of an MBARP, then $\vec{o}^T$ will be the allocation. After all, if $\left(r_1, \dots, r_m\right)$ are the reserve prices of an arbitrary MBARP, then it will always be the case that \[\sum_{i = 1}^n v_i^t\left(o_i^T\right) + \sum_{j \not \in T} r_j \geq \sum_{i = 1}^n v_i^t\left(o_i'\right) + \sum_{j \not \in T} r_j\] for any allocation $\vec{o}' \in \mathcal{O}_T$ by definition of $\vec{o}^T.$

Now, consider an MBARP parameterized by $\left(c, r_1, \dots, r_m\right)$. The allocation will be \[\vec{o}^T = \text{argmax}\left\{\sum_{i = 1}^n v_i^t\left(o_i^{[m]}\right) + c,\left\{ \sum_{i = 1}^n v_i^t\left(o_i^T\right) + \sum_{j \not\in T} r_j \right\}_{T \not= [m]}\right\}.\]

For any $T \subseteq [m]$, let $R_T^{\vec{v}^t}$ be the subset of $\R^{m+1}$ such that if an MBARP is parameterized by $\left(c, r_1, \dots, r_m\right) \in R_T^{\vec{v}^t}$, then the allocation of the MBARP on $\vec{v}^t$ is $\vec{o}^T$. This means that if $T \not= [m]$ \begin{align*}
\sum_{i = 1}^n v_i^t\left(o_i^T\right) + \sum_{j \not \in T} r_j &\geq \sum_{i = 1}^n v_i^t\left(o_i^{T'}\right) + \sum_{j \not \in T'} r_j & \forall T' \not\in\{ T, [m]\} \text{ and}\\
\sum_{i = 1}^n v_i^t\left(o_i^T\right) + \sum_{j \not \in T} r_j &\geq \sum_{i = 1}^n v_i^t\left(o_i^{[m]}\right) +
c. \end{align*} In other words, $\left(c, r_1, \dots, r_m\right) \in R_T^{\vec{v}^t}$ if and only if it falls in the intersection of $2^m-1$ halfspaces:
\begin{align*}
\sum_{j \not \in T} r_j - \sum_{j \not \in T'} r_j &\geq \sum_{i = 1}^n v_i^t\left(o_i^{T'}\right) -  v_i^t\left(o_i^T\right) & \forall T' \not\in\{ T, [m]\}\\
\sum_{j \not \in T} r_j - c &\geq \sum_{i = 1}^n v_i^t\left(o_i^{[m]}\right)  -  v_i^t\left(o_i^T\right). \end{align*}

Similarly, if $T = [m]$, it is not hard to see that we can write $R_T^{\vec{v}^t}$ as the intersection of $2^m-1$ halfspaces.

We can also analyze the allocation of an MBARP parameterized by $\left(c, r_1, \dots, r_m\right)$ without Bidder $i$'s participation for any $i \in [n]$, which we need to do in order to analyze the revenue function. To this end, for all $T \subseteq [m]$, let $\mathcal{O}_{T_{-i}}$ be the set of all allocations where exactly the elements of $T$ are allocated to all of the bidders except $i$, and let \[\vec{o}^{T_{-i}} = \underset{\vec{o} \in \mathcal{O}_{T_{-i}}}{\text{argmax}} \left\{\sum_{j \not= i} v_j^t\left(o_j\right)\right\}.\] Again, regardless of the reserve prices, if $T$ consists of the items allocated by an MBARP without Bidder $i$'s participation, then $\vec{o}^{T_{-i}}$ will be the allocation. Now, for an MBARP parameterized by $(c, r_1, \dots, r_m)$ without Bidder $i$'s participation, the allocation will be \[\vec{o}^{T_{-i}} = \text{argmax}\left\{\sum_{j \not = i} v_j^t\left(o_j^{[m]_{-i}}\right) + c,\left\{ \sum_{j \not = i} v_j^t\left(o_j^{T_{-i}}\right) + \sum_{\ell \not\in T} r_\ell \right\}_{T \not= [m]}\right\}.\]

For any $T \subseteq [m]$, let $R_{T_{-i}}^{\vec{v}^t}$ be the subset of $\R^{m+1}$ such that if an MBARP is parameterized by $\left(c, r_1, \dots, r_m\right) \in R_{T_{-i}}^{\vec{v}^t}$, then the allocation of the MBARP without Bidder $i$'s partitipation on $\vec{v}$ is $\vec{o}^{T_{-i}}$. This means that if $T \not= [m]$, then 
just as before, $\left(c, r_1, \dots, r_m\right) \in R_{T_{-i}}^{\vec{v}^t}$ if and only if it falls in the intersection of $2^m-1$ halfspaces:
\begin{align*}
\sum_{\ell \not \in T} r_\ell - \sum_{\ell \not \in T'} r_\ell &\geq \sum_{j \not= i} v_j\left(o_j^{T'_{-i}}\right) - \sum_{j \not= i} v_j\left(o_j^{T_{-i}}\right) & \forall T' \not\in\{ T, [m]\}\\
\sum_{\ell \not \in T} r_\ell - c &\geq \sum_{j \not= i} v_j\left(o_j^{[m]_{-i}}\right) - \sum_{j \not= i} v_j\left(o_j^{T_{-i}}\right). \end{align*}

Similarly, if $T = [m]$, we can write $R_{T_{-i}}^{\vec{v}}$ as the intersection of $2^m-1$ halfspaces.

Clearly, $\left\{R_T^{\vec{v}^t}\right\}_{T \subseteq [m]}$ partition $\R^{m+1}$, since there will always be some allocation of an MBARP parameterized by an arbitrary point in $\R^{m+1}$. Similarly, $\left\{R_{T_{-i}}^{\vec{v}^t}\right\}_{T \subseteq [m]}$ partition $\R^{m+1}$ for every $i \in [n]$.

Now, suppose \[\left(c, r_1, \dots, r_m\right) \in R_{T^0}^{\vec{v}^t} \bigcap R_{T_{-1}^1}^{\vec{v}^t} \bigcap \cdots \bigcap R_{T_{-n}^n}^{\vec{v}^t} = R\] for some $T^0, T^1, \dots, T^n \subseteq [m]$. We show that $rev_{\vec{v}^t}\left(c, r_1, \dots, r_m\right)$ is linear on $R$ by splitting the analysis into four cases.

\begin{enumerate}
\item If $T^0, T^1, \dots, T^n \not= [m]$ we can write
\begin{equation}\label{eq:case1} rev_{\vec{v}^t}\left(c, r_1, \dots, r_m\right) = \sum_{i = 1}^n \left[ \sum_{j \not= i}  \left(v_j^t\left(o_j^{T^i_{-i}}\right) - v_j^t\left(o_j^{T^0}\right)\right) + \sum_{\ell \not\in T^i} r_\ell - \sum_{\ell \not\in T^0} r_\ell \right].\end{equation}
\item If $T^0 \not= [m]$ and $T^i = [m]$ for some $i \in \{1, \dots, n\}$, then we replace the summand of Equation~\ref{eq:case1} indexed by $i$ with \[\sum_{j \not= i} \left(v_j^t\left(o_j^{T^i_{-i}}\right) - v_j^t\left(o_j^{T^0}\right) \right) +c - \sum_{\ell \not\in T^0} r_\ell.\]
\item If $T^0 = [m]$ and $T_1, \dots, T_n \not= [m]$, then \begin{equation}\label{eq:case2}rev_{\vec{v}^t}\left(c, r_1, \dots, r_m\right) = \sum_{i = 1}^n \left[ \sum_{j \not= i} \left(v_j^t\left(o_j^{T^i_{-i}}\right) - v_j^t\left(o_j^{T^0}\right)\right) + \sum_{\ell \not\in T^i} r_\ell -c \right].\end{equation}
\item If $T^0 = [m]$ and $T^i = [m]$ for some $i \in \{1, \dots, n\}$, then we replace the summand of Equation~\ref{eq:case2} indexed by $i$ with \[\sum_{j \not= i} \left(v_j^t\left(o_j^{T^i_{-i}}\right) - v_j^t\left(o_j^{T^0}\right)\right).\]
\end{enumerate}

In all of these cases, $rev_{\vec{v}^t}\left(c, r_1, \dots, r_m\right)$ is a linear function over  \[\left(c, r_1, \dots, r_m\right) \in R_{T^0}^{\vec{v}^t} \bigcap R_{T_{-1}^1}^{\vec{v}^t} \bigcap \cdots \bigcap R_{T_{-n}^n}^{\vec{v}^t} = R.\]

To summarize, we fixed $\vec{v}^t \in \sample$ and introduced $n+1$ partitions of $\R^{m+1}$. Each partition is made up of $2^m$ cells and each cell is defined as the intersection of $2^m-1$ halfspaces. If we restrict the domain of the revenue function to the intersection of any $n+1$ cells, one from each of the $n+1$ partitions, then the revenue function on that restricted domain is linear and therefore, there is one subregion where $rev_{\vec{v}^t}\left(c, r_1, \dots, r_m\right)$ exceeds its target revenue and one subregion where it does not.

One generous upper bound on the number of different regions induced by taking the intersection of any $n+1$ cells, one from each of the $n+1$ partitions, is the number of different regions induced by the $(n+1)2^m\left(2^m-1\right)$ total hyperplanes. This is at most $(m+1)\left((n+1)2^m\left(2^m-1\right)\right)^{m+1} \leq (m+1)\left((n+1)4^m\right)^{m+1}$ because the number of regions induced by $k$ hyperplanes in $\R^d$ is at most $\sum_{i = 1}^d{k \choose i} \leq dk^d.$ Again, if we restrict the domain of the revenue function to any of these induced regions, the revenue function will be linear. As we saw in cases (1)-(4) of our case analysis, depending on the region, the revenue function will take a specific linear form.
For a given region $R$, denote this specific linear form of the revenue function on by $R$ as $rev_{\vec{v}^j}|_R$. With this in mind, we define one more hyperplane per region: $rev_{\vec{v}^j}|_R > z^j$. In total, this contributes at most $(m+1)\left((n+1)4^m\right)^{m+1}$ more hyperplanes, since this is the maximum number of induced regions on $\R^{m+1}$. We are therefore left with at most $\alpha = (m+1)\left((n+1)4^m\right)^{m+1} + (n+1)2^m\left(2^m-1\right) = O\left(mn^m8^{m^2}\right)$ total hyperplanes per valuation vector function $\vec{v}^j \in \sample$.

If we merge all of the $N$ sets of $\alpha$ hyperplanes, $\R^{m+1}$ will be split into at most $(m+1)(N\alpha)^{m+1}$ regions, each of which induces a specific binary labeling of $\sample$ (whether or not $\vec{v}^j$ exceeds its target revenue). Therefore, it must be that $2^N \leq (m+1)(N\alpha)^{m+1}$, so $N = O(m \log \alpha) = O\left(m^3 \log n\right)$.
\hfill \text{ }
\end{proof}

We may now use this result to prove Theorem~\ref{thm:main_sample_MBARP}.

\begin{namedtheorem}[\ref{thm:main_sample_MBARP}]
The sample complexity of uniform convergence over the class of $n$-bidder, $m$-item MBARPs with item-specific reserve prices is \[N = O\left(\left(\frac{U}{\epsilon}\right)^2\left( m^3 \log n \log\frac{U}{\epsilon} + \log \frac{1}{\delta}\right)\right).\]
\end{namedtheorem}

\begin{proof}
This follows from Theorem~\ref{thm:pseudo} and Theorem~\ref{thm:MBARPupper}.
\end{proof}

\subsection{Bundle Reserve Prices Lower Bound}\label{sec:bundle_lower}


In this section, we justify our choice to concentrate on MBARPs with item-specific reserve prices. In particular, we prove that no algorithm can learn over the class of $n$-bidder, $m$-item MBARP revenue functions with bundle-specific reserve prices using sample complexity $o\left(4^m/\sqrt{m}\right)$.

As in the proof of Theorem~\ref{thm:AMAlower}, we construct a special set $V$ of valuation functions. In this case, $V$ is a set of single-bidder, $m$-item valuation functions and $|V| = \Omega\left(4^m/\sqrt{m}\right)$. We then show that for any subset $H$ of $V$, there exists a setting of the bundle reserve prices that has high revenue over valuation functions in $H$, but low revenue on the valuation functions in $V \setminus H$. We describe $V$ more formally in Theorem~\ref{thm:bundle_res_price}. As we show in Remark~\ref{rmk:bundle_res_extended}, this construction can trivially be extended to a set of $n$-bidder, $m$-item valuation functions of the same size. Then, as in Theorem~\ref{thm:AMAlower}, this immediately implies hardness for learning over the uniform distribution on $V$. We provide the construction of $V$ and, given the parallel proof structure, we refer the reader to Theorem~\ref{thm:AMAlower} to see how this implies hardness of learning.

We now present the construction of the set of valuation functions $V$.

\begin{theorem}\label{thm:bundle_res_price}
For any $m \geq 2$, there exists a set of $N = \Omega\left(\frac{4^m}{\sqrt{m}}\right)$ single-bidder, $m$-item valuation function vectors $V = \left\{\vec{v}^1, \dots, \vec{v}^N\right\}$ such that for any $H \subseteq V$, there exists a set of monotone bundle reserve prices such that the resulting auction has revenue 0 on $\vec{v}^i$ if $\vec{v}^i \in H$ and revenue $1-\gamma$ on $\vec{v}^i$ if $\vec{v}^i \not\in H$, for any $\gamma \in (0,1)$.
\end{theorem}

\begin{proof}
We define the set $V = \left\{\vec{v}^1, \dots, \vec{v}^N\right\}$ of single-bidder valuation functions, where $\vec{v}^j = \left(v_1^j\left(b_1\right), \dots, v_1^j\left(b_{2^m}\right)\right)$. Assume for now that $m$ is even, and let $\tilde{b}_1, \dots, \tilde{b}_N$ be a fixed ordering of the subsets of  $2^{[m]}$ of size $m/2$, so $N = \Omega\left(\frac{4^m}{\sqrt{m}}\right)$. Let $\vec{v}^\ell$ for $\ell \in [N]$ be defined as follows.

\[v_1^\ell\left(b_i\right) = \begin{cases} 0 &\text{if } \left|b_i\right| < m/2\\
1 & \text{if } b_i = \tilde{b}_\ell\\
0 & \text{if } \left|b_i\right| = m/2 \text{ and } b_i \not= \tilde{b}_\ell\\
1 & \text{if } \left| b_i \right| > m/2
\end{cases}.\]

We claim that for any $H \subseteq V$, there exists a set of monotone bundle reserve prices\newline $\left\{r\left(b_1\right), \dots, r\left(b_{2^m}\right)\right\}$ such that the resulting auction has 0 revenue on all valuation functions $\vec{v}^i$ such that $\vec{v}^i \not\in H$ and $1-\gamma$ revenue on all valuation functions $\vec{v}^i \in H$. The reserve prices are defined as follows: \[v_0\left(b_i\right) = r\left(b_i\right) = \begin{cases} 0 & \text{if } \left|b_i\right| < m/2\\
1-\gamma &\text{if } b_i = \tilde{b}_\ell^c \text{ and } \vec{v}^\ell \not \in H\\
0 &\text{if }  b_i = \tilde{b}_\ell^c \text{ and } \vec{v}^\ell \in H\\
1-\gamma & \left| b_i \right| > m/2 \end{cases}.\]

Regardless of whether or not $\vec{v}^\ell$ is in $H$, $\left(\tilde{b}_\ell, \tilde{b}_\ell^c\right) = \text{argmax}_{\vec{o} \in \mathcal{O}} \left\{v_0\left(o_0\right) + v_1^\ell\left(o_1\right)\right\}$ and $1-\gamma = \max_{b_i \in 2^{[m]}} \left\{v_0\left(b_i\right)\right\}$. Therefore, Bidder 1 pays \[1-\gamma - v_0\left(\tilde{b}_\ell^c\right) = \begin{cases} 1-\gamma &\text{if } \vec{v}^\ell \in H\\
0 &\text{if } \vec{v}^\ell\not \in H \end{cases}.\]
\hfill \text{ }
\end{proof}

\begin{remark}\label{rmk:bundle_res_extended}
For any $m \geq 2, n \geq 1$, there exists a set of $N = \Omega\left(4^m/\sqrt{m}\right)$ $n$-bidder valuation function vectors $V = \left\{\vec{v}^1, \dots, \vec{v}^N\right\}$ such that for any $H \subseteq V$, there exists a set of monotone bundle reserve prices such that the resulting auction has revenue 0 on $\vec{v}^i$ if $\vec{v}^i \not\in H$ and revenue $1-\epsilon$ on $\vec{v}^i$ if $\vec{v}^i \in H$.
\end{remark}

This follows simply by setting $v_1^\ell$ as in the proof of Theorem~\ref{thm:bundle_res_price} and setting $v_j^\ell\left(b_i\right) = 0$ for all $\ell \in [N], i \in \left[2^m\right], j \not= 1$.

%% file: conclusion.tex
\section{Conclusion}

In this paper, we proved strong bounds on the sample complexity of uniform convergence for the well-studied and standard auction families that constitute the hierarchy of deterministic combinatorial auctions. We thereby answered a crucial question in the study of (automated) mechanism design: how to relate the performance of the mechanisms in the search space over the input samples to their expectation over the underlying---unknown---distribution. Specifically, for a fixed class of auctions, we determine the sample complexity necessary to ensure that with high probability, for any auction in that class, the average revenue over the sample is close to the expected revenue with respect to the underlying, unknown distribution over bidders' valuations. Our bounds apply to any algorithm that finds an optimal or approximately optimal auction over an input sample, and therefore to any automated mechanism design algorithm. Moreover, our results and analyses are of interest from a learning theoretic perspective because the function classes which make up the hierarchy of deterministic combinatorial auctions diverge significantly from well-understood hypothesis classes typically found in machine learning.

\bigskip

\textbf{Acknowledgments.} This work was supported in part by NSF grants CCF-1451177, CCF-1422910, a Sloan Research Fellowship, and a Microsoft Research Faculty Fellowship.

%% file: appendix.tex
\section{Proofs from Section~\ref{sec:intro}}\label{app:intro_app}

\begin{proof}[Proof of Theorem~\ref{thm:approx_bounds}]
First, let $\epsilon= 2R_N(\mathcal{H}) + c\sqrt{\frac{2\ln(4/\delta)}{N}}$. For ease of notation, for any $h \in \mathcal{H}$, let $L_\sample(h) = \frac{1}{N} \sum_{i = 1}^N \ell\left(h,x_i\right)$ and $L_{\mathcal{D}}(h) = \mathbb{E}_{x \sim \mathcal{D}} [\ell(h,x)].$ Suppose that $h^*$ is the optimal hypothesis in $\mathcal{H}$ (i.e. it minimizes $L_\mathcal{D}(h)$, the expected loss over the distribution $\mathcal{D}$), $\hat{h}$ is the empirical risk minimizer (i.e. it minimizes $L_\mathcal{S}(h)$, the average loss over the sample $\sample$), and $\tilde{h}$ is a hypothesis such that $L_\sample\left(\tilde{h}\right) -L_\sample\left(\hat{h}\right) \leq \rho$ for some $\rho > 0$. Then with probability at least $1-\delta$,
\begin{align}
 L_{\mathcal{D}}\left(\tilde{h}\right) - \epsilon &\leq   L_{\sample}\left(\tilde{h}\right)\label{ineq:1a} \\
&\leq L_{\sample}\left(\hat{h}\right) + \rho\label{ineq:2a}\\ 
&\leq  L_{\sample}\left({h^*}\right) + \rho\label{ineq:3a} \\
& \leq L_{\mathcal{D}}\left({h^*}\right) +  c \sqrt{\frac{2\ln(4/\delta)}{2N}}+\rho.\label{ineq:4a}
\end{align}

Inequality~\ref{ineq:1a} follows from Equation standard Rademacher complexity uniform convergence bounds: with probability at least $1-\delta/2$, $L_{\mathcal{D}}\left(\tilde{h}\right) - L_{\sample}\left(\tilde{h}\right) \leq \epsilon$ (see, for example, \cite{Shalev14:Understanding}). Inequality~\ref{ineq:2a} follows from the fact that $L_\sample\left(\tilde{h}\right) \leq L_\sample\left(\hat{h}\right) + \rho$.  Inequality~\ref{ineq:3a} follows because $\hat{h}$ is the empirical risk minimizer (i.e. it minimizes $L_\mathcal{S}(h)$). Finally, inequality~\ref{ineq:4a} is a result, again, of Hoeffding's inequality, which guarantees that with probability at least $1-\delta/2$, $L_{\sample}\left({h^*}\right) \leq L_{\mathcal{D}}\left({h^*}\right) +  c \sqrt{\frac{2\ln(4/\delta)}{2N}}$.

Rearranging, we get that \[L_{\mathcal{D}}\left(\tilde{h}\right) - L_{\mathcal{D}}\left({h^*}\right) \leq \epsilon + c \sqrt{\frac{2\ln(4/\delta)}{2N}} + \rho,\] as claimed.

Next, suppose that $\tilde{h}$ is a hypothesis such that $L_\sample\left(\tilde{h}\right) \leq (1+\alpha)L_\sample\left(\hat{h}\right).$ We similarly can deduce that with probability at least $1-\delta$,
\begin{align}
 L_{\mathcal{D}}\left(\tilde{h}\right) - \epsilon &\leq   L_{\sample}\left(\tilde{h}\right)\label{ineq:1} \\
&\leq (1+\alpha) L_{\sample}\left(\hat{h}\right)\label{ineq:2}\\ 
&\leq  (1+\alpha) L_{\sample}\left({h^*}\right)\label{ineq:3} \\
& \leq (1+\alpha) \left(L_{\mathcal{D}}\left({h^*}\right) +  c \sqrt{\frac{2\ln(4/\delta)}{2N}}\right).\label{ineq:4}
\end{align}

Rearranging, we get that
\[L_{\mathcal{D}}\left(\tilde{h}\right) \leq \epsilon + (1+\alpha) L_{\mathcal{D}}\left({h^*}\right) +  (1+\alpha) c\sqrt{\frac{2\ln(4/\delta)}{2N}},\] which means that \[L_{\mathcal{D}}\left(\tilde{h}\right) - L_{\mathcal{D}}\left({h^*}\right) \leq \epsilon + (1+\alpha) c\sqrt{\frac{2\ln(4/\delta)}{2N}} + \alpha L_{\mathcal{D}}\left({h^*}\right),\] as desired.

We remark that inequalities~\ref{ineq:1a}-\ref{ineq:4} could be tight in the worst case, so both bounds are tight.
\hfill \text{ }
\end{proof}

\section{Connection between pseudo-dimension and Rademacher complexity}\label{app:techniques_app}

In order to show that $\widehat{\mathcal{R}}_\sample (\mathcal{F}) = \tilde{O}\left(\sqrt{d_{\mathcal{F}}}/N\right)$, we connect Rademacher complexity to pseudo-dimension by way of the learning-theoretic concept of \emph{covering numbers}, which are defined as follows.

\begin{definition}[Coverage number]
Let $A \subset \R^N$ be a set of vectors. We define $N_p(r,A)$ to be the cardinality of the smallest set $A' \subset \R^N$ such that for all $a \in A$, there exists $a' \in A'$ such that $||a-a'||_p \leq r$. We say that such an $A'$ $r-$covers $A$ in the $\ell_p$ norm.
\end{definition}

Along with pseudo-dimension and Rademacher complexity, coverage numbers are another tool for measuring the richness of a class of functions, and thereby deriving sample complexity bounds. We can relate the above definition to a class of functions $\mathcal{F}$ by defining \[N_p(r, \mathcal{F}, N) = \max\{N_p(r, \mathcal{F}|_{\sample}) \ | \ \sample \in X^N\},\] where for $\sample = (x_1, \dots, x_N) \in X^N$, $\mathcal{F}|_\sample = \{(f(x_1), \dots, f(x_N) \ | \ f \in \mathcal{F}\}$. Notice that $\mathcal{F}|_\sample \subset \R^N$.

\begin{claim}\label{claim:pseudo_rad_relation}
Let $\mathcal{F}$ be a class of real-valued functions with pseudo-dimension $d_{\mathcal{F}}$ and range in $[0,c]$ for some $c \in \R$, and let $\sample = \left\{x_1, \dots, x_N\right\}$ be a sample of size $N$. Then \[\widehat{\mathcal{R}}_\sample (\mathcal{F}) \leq \frac{6\bar{c}}{N}\left(\sqrt{d_{\mathcal{F}} \log\frac{eNc}{\bar{c}d_{\mathcal{F}}}} + 2\sqrt{d_{\mathcal{F}}}\right),\] where $\bar{c} = \max_{f \in \mathcal{F}}\sqrt{\sum_{i = 1}^Nf(x_i)^2}.$
\end{claim}

\begin{proof}
Let $\sample$ be a subset of $X$ of size $N$, and let $A = \mathcal{F}|_\sample$. By definition, $N_2(r,A) \leq N_2(r,\mathcal{F}, N)$ and from Lemma 10.5 of \cite{Anthony09:Neural}, which states that $N_2(r', \mathcal{F}', N') \leq N_\infty(r', \mathcal{F}', N')$ for any $\mathcal{F}',$ $r',$ and $N'$, we know that $N_2(r,A) \leq N_\infty(r,\mathcal{F}, N)$. \cite{Anthony09:Neural} also prove that for any $\mathcal{F}',$ $r',$ and $N'$, where $\mathcal{F}'$ has pseudo-dimenion $d_{\mathcal{F}'}$ and maps to the bounded interval $[0,c']$ for some $c' \in \R$, $N_\infty(r', \mathcal{F}', N')$ is upper bounded by $\sum_{i = 1}^{d_{\mathcal{F}'}} {N' \choose i} \left(\frac{c'}{r'}\right)^i$, which, in turn, is less than $\left(\frac{eN'c'}{r'd_{\mathcal{F}'}}\right)^{d_{\mathcal{F}'}}$ for $N' \geq d_{\mathcal{F}'}$. Putting this all together, for our original function class, we can guarantee that \[N_2(r,A) \leq \left(\frac{eNc}{rd_{\mathcal{F}}}\right)^{d_{\mathcal{F}}}\] for $N \geq d_{\mathcal{F}}$.

We use this fact to bound the empirical Rademacher complexity of $\mathcal{F}$ by calling on Lemma 27.5 of \cite{Shalev14:Understanding}, which states that for any $A' \subseteq \R^{N'}$, if there are $\alpha,\beta>0$ such that for any $k \geq 1$, $\sqrt{\log(N_2(\bar{c}2^{-k}),A'))} \leq \alpha + \beta k,$ then \[\mathcal{R}(A')= \frac{1}{N'} \mathbb{E}_{\sigma}\left[ \sup_{a \in A'}\sum_{i = 1}^{N'} \sigma_i a_i \right] \leq \frac{6\bar{c}}{N'}(\alpha + 2\beta),\] where $\bar{c} = \min_{\bar{a}}\max_{a \in A'}||a - \bar{a}||_2$.

For our set $A = \mathcal{F}|_{\sample}$, let $\bar{a}$ be a minimizer of the objective function given in the definition of $\bar{c}$. Since Rademacher complexity is invariant under translation,\footnote{For all $D \subseteq \R^N, b \in \R^N$, $\mathcal{R}(D) = \mathcal{R}(\{d+b \ | \ d \in D\})$.} we can analyze the Rademacher complexity assuming that $\bar{a} = 0$. Moreover, since $A = \mathcal{F}|_{\sample}$, we know that $\bar{c} = \max_{a \in A}||a||_2  = \max_{f \in \mathcal{F}}\sqrt{\sum_{i = 1}^Nf(x_i)^2}.$

Now we derive the required $\alpha,\beta$ as follows. \begin{align*}
\sqrt{\log(N_2(\bar{c}2^{-k},A))} &\leq \sqrt{\log(N_\infty(\bar{c}2^{-k},\mathcal{F},N))}\\
&\leq \sqrt{d_{\mathcal{F}} \log\frac{eNc}{\bar{c}2^{-k}d_{\mathcal{F}}}}\\
& = \sqrt{d_{\mathcal{F}} \log\frac{eNc2^k}{\bar{c}d_{\mathcal{F}}}}\\
& = \sqrt{d_{\mathcal{F}} \left(\log\frac{eNc}{\bar{c}d_{\mathcal{F}}} + \log2^k\right)}\\
& \leq \sqrt{d_{\mathcal{F}} \log\frac{eNc}{\bar{c}d_{\mathcal{F}}}} + \sqrt{d_{\mathcal{F}}k}\\
&\leq \sqrt{d_{\mathcal{F}} \log\frac{eNc}{\bar{c}d_{\mathcal{F}}}} + \sqrt{d_{\mathcal{F}}}k.\\
\end{align*}

Therefore $\alpha = \sqrt{d_{\mathcal{F}} \log\frac{eNc}{\bar{c}d_{\mathcal{F}}}}$ and $\beta = \sqrt{d_{\mathcal{F}}}$, so \[\mathcal{R}(\mathcal{F}|_\sample) = \mathcal{R}_\sample (\mathcal{F}) \leq \frac{6\bar{c}}{N}\left(\sqrt{d_{\mathcal{F}} \log\frac{eNc}{\bar{c}d_{\mathcal{F}}}} + 2\sqrt{d_{\mathcal{F}}}\right).\]
\hfill \text{ }
\end{proof}

\section{Proofs from Section~\ref{sec:AMAupper}}\label{app:AMAupper}

Here, we provide the lemmas referred to in the proof of Theorem~\ref{thm:AMAupper}. In particular, we bound the Rademacher complexity of the function classes consisting of the simpler components we broke the AMA revenue function into: $rev_{A,1},\dots, rev_{A,n+1}$. Recall that \[rev_{A,j}(\vec{v}) = \frac{1}{w_j} \phi_{A, -j}(\vec{v}),\] where \[ \phi_{A,-j}(\vec{v}) = \max_{\vec{o}_i \in \mathcal{O}} \left\{\sum_{\ell \not= j} w_\ell v_\ell(o_{i,\ell}) + \lambda_i\right\}\] and \[\mathcal{L}_j = \{rev_{A,j} \ | \ A = (w_1, \dots, w_n,\lambda_1,\dots, \lambda_{(n+1)^m}), H_{\underline{w}}\leq|w_i|\leq H_{\overline{w}}, |\lambda_i| \leq H_\lambda\}.\] It is helpful to note that $rev_{A,j}$ is a weighted version of what the social welfare would have been if Bidder $j$ had not participated in the auction. In Lemma~\ref{lemma:Lj}, we bound the Rademacher complexity of $\mathcal{L}_j$ for $j \in [n]$.

To complete the analysis, we need to analyze the Rademacher complexity of $\mathcal{L}_{n+1}$, where \[\mathcal{L}_{n+1} = \{rev_{A,n+1} \ | \ A = (w_1, \dots, w_n,\lambda_1,\dots, \lambda_{(n+1)^m}), H_{\underline{w}}\leq|w_i|\leq H_{\overline{w}}, |\lambda_i| \leq H_\lambda\},\]
\[rev_{A,n+1}(\vec{v}) = - \sum_{i = 1}^{(n+1)^m} \left(\sum_{j = 1}^n \frac{1}{w_j} \sum_{\ell\not= j} w_\ell v_\ell(o_{i, \ell}) + \lambda_i\right) \mathbbm{1}_{\vec{o}_i = \vec{o}_A^*(\vec{v})},\] and \[\vec{o}^*_A(\vec{v}) = \underset{\vec{o}_i \in \mathcal{O}}{\text{argmax}} \left\{\sum_{j = 1}^n w_jv_j(o_{i,j}) + \lambda_i\right\},\] As noted in the main body of the paper, $rev_{A,n+1}$ is the amount of revenue subtracted out in order to ensure that the resulting auction is strategy-proof. We bound the Rademacher complexity of $\mathcal{L}_{n+1}$ in Lemma~\ref{lemma:Ln+1}.

These bounds can then be combined as described in the proof of Theorem~\ref{thm:AMAupper}.

\begin{lemma}\label{lemma:Lj}
For $j \in [n]$, \[\mathcal{R}_N(\mathcal{L}_j) = O\left(\frac{n^m \hat{H}_v(nH_{\overline{w}} + H_\lambda)}{H_{\underline{w}}}\sqrt{\frac{m\log n}{N}}\right).\]
\end{lemma}

\begin{proof}
Let \[\Phi_j = \{\phi_{A,-j} \ | \ A = (w_1, \dots, w_n,\lambda_1,\dots, \lambda_{(n+1)^m}), H_{\underline{w}}\leq|w_i|\leq H_{\overline{w}}, |\lambda_i| \leq H_\lambda\}.\]

Now, we can write each function $\sum_{\ell \not= j} w_\ell v_\ell(o_{i,\ell}) + \lambda_i$ as a linear function $h_{A,j}^i$ from $\R^{n2^m+1}$ to $\R$ as follows. Let $h_{A,j}^i(\vec{v}, 1) = (\vec{v}, 1) \cdot \vec{a}_{A,j}^i$, where \[\vec{a}_{A,j}^i[\ell] = \begin{cases} w_t &\text{ if } \ell = 2^m(t-1) + \sigma(i,t) \text{ and } t \not = j\\
\lambda_i &\text{ if } \ell = n2^m + 1\\
0 & \text{ otherwise }
\end{cases}.\]

Notice that $||(\vec{v}, 1)||_\infty \leq \max\{H_v, 1\} = \hat{H}_v$ and $||\vec{a}_{A,j}^i||_1 \leq nH_{\overline{w}} + H_\lambda.$ Let \[\mathcal{H}_j^i = \left\{h^i_{A,j} \ | \ A = (w_1, \dots, w_n,\lambda_1,\dots, \lambda_{(n+1)^m}), H_{\underline{w}}\leq|w_i|\leq H_{\overline{w}}, |\lambda_i| \leq H_\lambda\right\}.\] Using the $L_1$-norm Rademacher complexity bound for linear functions, we have that for all $i \in [2^m]$, \[\mathcal{R}_N(\mathcal{H}_1^i) \leq \hat{H}_v(nH_{\overline{w}} + H_\lambda)\sqrt{\frac{2\log(n2^{m}+1)}{N}}.\]

Now, for two hypothesis sets $H$ and $H'$ of functions mapping from $X$ to $\R$, \begin{equation}\label{eq:max}\mathcal{R}_N(\left\{\max(h,h') \ | \ h \in H, h' \in H'\right\}) \leq \mathcal{R}_N(H) + \mathcal{R}_N(H'),\end{equation} where $\max(h,h')$ denotes the function $x \mapsto \max(h(x),h'(x))$ \cite{Mohri12:Foundations}. Therefore, \[\mathcal{R}_N(\Phi_j) \leq (n+1)^m \mathcal{R}_N(\mathcal{H}_j^1) \leq (n+1)^m\hat{H}_v(nH_{\overline{w}} + H_\lambda)\sqrt{\frac{2\log(n2^{m}+1)}{N}},\]
which means that, \begin{align*}\mathcal{R}_N(\mathcal{L}_j) = \frac{1}{H_{\underline{w}}} \mathcal{R}_N(\Phi_j) &\leq \frac{(n+1)^m \hat{H}_v(nH_{\overline{w}} + H_\lambda)}{H_{\underline{w}}}\sqrt{\frac{2\log(n2^{m}+1)}{N}}\\
&= O\left(\frac{n^m \hat{H}_v(nH_{\overline{w}} + H_\lambda)}{H_{\underline{w}}}\sqrt{\frac{m\log n}{N}}\right).\end{align*}
\hfill \text{ }
\end{proof}

\begin{lemma}\label{lemma:Ln+1}
\[\mathcal{R}_N(\mathcal{L}_{n+1}) =O\left( \frac{n^{m+2} \left(H_{\overline{w}}H_v+ H_\lambda\right)}{H_{\underline{w}}}\sqrt{\frac{m\log n}{N}}\left(\frac{n\hat{H}_v\left(nH_{\overline{w}} + H_\lambda\right)}{H_{\underline{w}}} + \sqrt{n^m \log N}\right)\right).\]
\end{lemma}

\begin{proof}
We use the following lemma, which is similar to Lemma 3 in \cite{DeSalvo15:Learning}, to bound the Rademacher complexity of \[\mathcal{L}_{n+1} = \{rev_{A,n+1} \ | \ A = (w_1, \dots, w_n,\lambda_1,\dots, \lambda_{(n+1)^m}), H_{\underline{w}}\leq|w_i|\leq H_{\overline{w}}, |\lambda_i| \leq H_\lambda\}.\]

\begin{lemma}\label{lemma:mult_functions}
Let $\mathcal{F}$ be a family of functions mapping $\mathcal{X}$ to $[-c,c]$, let $\mathcal{G}$ be a a family of functions mapping $\mathcal{X}$ to $\{0,1\}$, and let $\mathcal{H} = \{fg \ | \ f \in \mathcal{F}, g \in \mathcal{G}\}$. Then \[\mathcal{R}_N(\mathcal{H})  \leq (c+1)(\mathcal{R}_N(\mathcal{F}) + \mathcal{R}_N(\mathcal{G})).\]
\end{lemma}

\begin{proof}[Proof of Lemma~\ref{lemma:mult_functions}]
Notice that for any $f \in \mathcal{F}, g \in \mathcal{G}$, we have that $fg = \frac{1}{4}[(f+g)^2 - (f-g)^2]$. For $x \in [-c,c+1]$, the function $x \mapsto \frac{1}{4}x^2$ is $\frac{1}{2}(c+1)$-Lipschitz. The same holds for $x \in [-c-1,c]$. Therefore, by Talagrand's lemma (e.g. \cite{Mohri12:Foundations}), we have that \[\widehat{\mathcal{R}}_S(\mathcal{H}) \leq \frac{1}{2}(c+1)[\widehat{\mathcal{R}}_S(\mathcal{F} + \mathcal{G}) + \widehat{\mathcal{R}}_S(\mathcal{F} - \mathcal{G})] \leq (c+1)(\widehat{\mathcal{R}}_S(\mathcal{F}) + \widehat{\mathcal{R}}_S(\mathcal{G})).\] Therefore, $\mathcal{R}_N(\mathcal{H})  \leq (c+1)(\mathcal{R}_N(\mathcal{F}) + \mathcal{R}_N(\mathcal{G})).$
\hfill \text{ }
\end{proof}

To use Lemma~\ref{lemma:mult_functions}, we first define a set of functions for each $i \in [(n+1)^m]$ \[\mathcal{F}_i = \left\{f_{A,i} \ | \ A = (w_1, \dots, w_n,\lambda_1,\dots, \lambda_{(n+1)^m}), H_{\underline{w}}\leq|w_j|\leq H_{\overline{w}}, |\lambda_j| \leq H_\lambda\right\},\] where \[f_{A,i}(\vec{v}) = \sum_{j = 1}^n \frac{1}{w_j} \sum_{\ell\not= j} w_\ell v_\ell(o_{i, \ell}) + \lambda_i.\]

As in the proof of Lemma~\ref{lemma:Lj}, we can write each $f_{A,i}$ as a linear function from $\R^{n2^{m}+1}$ to $\R$ as follows. Let $h_{A,i}(\vec{v},1) = (\vec{v},1)\cdot\vec{a}_{A,i}$, where \[\vec{a}_{A,i}[\ell] = \begin{cases} w_t\sum_{s \not= t} \frac{1}{w_s} & \text{if } \ell = 2^m(t-1) + \sigma(i,t)\\
\lambda_i \sum_{i = 1}^n \frac{1}{w_i} & \text{if }\ell = n2^{m} +1\\
0 &\text{otherwise}\end{cases}.\] Then $f_{A,i}(\vec{v}) = h_{A,i}(\vec{v},1)$. As before, we have that $||(\vec{v}, 1)||_\infty \leq \max\{H_v, 1\} = \hat{H}_v$. Moreover, $||\vec{a}_{A,i}||_1 \leq \frac{n}{H_{\underline{w}}}(nH_{\overline{w}} + H_\lambda).$

Using the $L_1$-norm Rademacher complexity bound for linear functions, we have that \begin{equation}\mathcal{R}_N(\mathcal{F}_i) \leq \frac{n\hat{H}_v}{H_{\underline{w}}}(nH_{\overline{w}} + H_\lambda)\sqrt{\frac{2\log(n2^{m}+1)}{N}}.\label{eq:F_i}\end{equation}

Now, we define a set of functions $\mathcal{G}_i$ for each $i \in [(n+1)^m]$ as \[\mathcal{G}_i = \left\{g_{A,i} \ | \ A = (w_1, \dots, w_n,\lambda_1,\dots, \lambda_{(n+1)^m}), H_{\underline{w}}\leq|w_j|\leq H_{\overline{w}}, |\lambda_j| \leq H_\lambda\right\},\] where $g_{A,i}(\vec{v}) = 1$ if and only if $i = \vec{o}_A^*(\vec{v})$, i.e. \[g_{A,i}(\vec{v}) = \begin{cases} 1 &\text{if } i = \underset{i \in [(n+1)^m]}{\text{argmax}} \left\{\sum_{j = 1}^n w_jv_j(o_{i,j}) + \lambda_i\right\}\\
0 &\text{otherwise}\end{cases}.\]

Notice that we can also write each function $g_{A,i}(\vec{v})$ as an intersection of $(n+1)^{m}-1$ binary-valued functions $\{c_\ell\}$, where $\ell \in [(n+1)^m]\setminus \{i\}$, as follows. \begin{equation} \label{eq:cell} c_\ell (\vec{v}) = \begin{cases} 1 &\text{if } \sum_{j = 1}^n w_jv_j(o_{i,j}) + \lambda_i \geq \sum_{j = 1}^n w_jv_j(o_{\ell,j}) + \lambda_\ell\\
0 &\text{otherwise}\end{cases}.\end{equation}

Indeed, $c_\ell(\vec{v}) = 1$ for all $\ell \not= i$ if and only if $g_{A,i}(\vec{v}) = 1$, i.e. \[i = \underset{i \in [(n+1)^m]}{\text{argmax}} \left\{\sum_{j = 1}^n w_jv_j(o_{i,j}) + \lambda_i\right\}.\] Each function $c_\ell$ can be written as a linear separator over $\R^{n2^{m}}$, so the VC dimension of $\{c_\ell\}$ is $n2^{m}+1$. This allows us to use Lemma 3.2.3 from \cite{Blumer89:Learnability} to bound the VC dimension of $\mathcal{G}_i$.

\begin{lemma}[Lemma 3.2.3 from \cite{Blumer89:Learnability}]\label{lemma:blumer}
Let $C$ be a concept class of finite VC dimension $d \geq 1$. For all $s \geq 1$, let $C_s \{\cap_{i = 1}^s c_i \ | \ c_i \in C, 1 \leq i \leq s\}$. Then for all $s \geq 1$, the VC dimension of $C_s$ is less than $2ds\log (3s)$.
\end{lemma}

Therefore, the VC dimension of $\mathcal{G}_i$ is less than $2(n2^{m}+1) (n+1)^m \log (3\cdot (n+1)^m) = O(mn^m \log n).$ By Corollary 3.1 in \cite{Mohri12:Foundations}, we have that \begin{equation}\mathcal{R}_N(\mathcal{G}_i) = O\left( \sqrt{\frac{mn^m \log n \log N}{N}}\right).\label{eq:G_i}\end{equation}

Putting Equations (\ref{eq:F_i}) and (\ref{eq:G_i}) together with Lemma \ref{lemma:mult_functions}, we conclude that if $\mathcal{H}_i = \{f_{A,i}g_{A,i} \ | \ f_{A,i} \in \mathcal{F}_i, g_{A,i} \in \mathcal{G}_i\}$, then \[\mathcal{R}_N(\mathcal{H}_i) = O\left(\left(\frac{n^2 (H_{\overline{w}}H_v+ H_\lambda)}{H_{\underline{w}}}\right)\left(\frac{n\hat{H}_v}{H_{\underline{w}}}(nH_{\overline{w}} + H_\lambda)\sqrt{\frac{m\log n}{N}} + \sqrt{\frac{mn^m \log n \log N}{N}}\right)\right).\] This follows from Lemma~\ref{lemma:mult_functions}, since the range of any function in $\mathcal{F}_i$ is $[0, n(n-1)(H_{\overline{w}}H_v + H_\lambda)/H_{\underline{w}}]$.

Finally, since \[rev_{A,n+1}(\vec{v}) = -\sum_{i = 1}^{(n+1)^m} f_{A,i}(\vec{v})g_{A,i}(\vec{v}),\] we have that 
\begin{align*}&\mathcal{R}_N(\mathcal{L}_{n+1})\\ = \text{ }&O\left( n^m \left(\frac{n^2 (H_{\overline{w}}H_v+ H_\lambda)}{H_{\underline{w}}}\right)\left(\frac{n\hat{H}_v}{H_{\underline{w}}}(nH_{\overline{w}} + H_\lambda)\sqrt{\frac{m\log n}{N}} + \sqrt{\frac{mn^m \log n \log N}{N}}\right)\right).\end{align*}
By rearranging terms, we get the desired result.
\hfill \text{ }
\end{proof}